\documentclass[journal,10pt]{IEEEtran}

\usepackage{algorithm}
\usepackage{algorithmic}
\usepackage{amssymb}
\usepackage{amsthm}
\usepackage{bm}
\usepackage{booktabs}
\usepackage[T1]{fontenc}
\usepackage{mathtools}
\usepackage{multibib}
\usepackage{multicol}
\usepackage{tabularx}
\usepackage{xcolor}

\ifodd 0 \newcommand{\com}[1]{{\color{red}#1}} \else \newcommand{\com}[1]{} \fi
\ifodd 0 \newcommand{\dtl}[1]{{\color{red}Details: #1}} \else \newcommand{\dtl}[1]{} \fi
\ifodd 0 \newcommand{\reva}[1]{{\color{blue}#1}} \else \newcommand{\reva}[1]{#1} \fi
\ifodd 0  \else  \fi

\newcites{appendix}{Additional References}

\DeclareMathOperator{\argmin}{argmin}
\DeclareMathOperator{\argmax}{argmax}
\newcommand{\Ex}{\mathbb{E}}
\newcommand{\In}{\mathbb{I}}
\renewcommand{\Pr}{\mathbb{P}}

\newtheorem{assumption}{Assumption}
\newtheorem{fact}{Fact}
\newtheorem{lemma}{Lemma}
\newtheorem{theorem}{Theorem}

\allowdisplaybreaks

\title{Thompson Sampling for Combinatorial Network Optimization in Unknown Environments}
\author{
	\IEEEauthorblockN{Alihan H\"uy\"uk, Cem Tekin,~\IEEEmembership{Senior Member,~IEEE}}
	\thanks{\copyright 2020 IEEE.  Personal use of this material is permitted.  Permission from IEEE must be obtained for all other uses, in any current or future media, including reprinting/republishing this material for advertising or promotional purposes, creating new collective works, for resale or redistribution to servers or lists, or reuse of any copyrighted component of this work in other works.}
	\thanks{A. H\"uy\"uk was with the Department of Electrical and Electronics Engineering, Bilkent University, Ankara 06800, Turkey. He is now with the Department of Applied Mathematics and Theoretical Physics, University of Cambridge, Cambridge CB3 0WA, UK (e-mail: ah2075@cam.ac.uk).}
    \thanks{C. Tekin is with the Department of Electrical and Electronics Engineering, Bilkent University, Ankara 06830, Turkey (e-mail: cemtekin@ee.bilkent.edu.tr).}
	\thanks{This work was supported in part by the Scientific and Technological Research Council of Turkey under Grant 215E342.}
	\thanks{A preliminary version of this work was presented in AISTATS 2019 \cite{huyuk2019analysis}.}
}

\begin{document}

\maketitle
\begin{abstract}
Influence maximization, adaptive routing, and dynamic spectrum allocation all require choosing the right action from a large set of alternatives. Thanks to the advances in combinatorial optimization, these and many similar problems can be efficiently solved given an environment with known stochasticity. In this paper, we take this one step further and focus on combinatorial optimization in unknown environments. We consider a very general learning framework called combinatorial multi-armed bandit with probabilistically triggered arms and a very powerful Bayesian algorithm called Combinatorial Thompson Sampling (CTS). Under the semi-bandit feedback model and assuming access to an oracle without knowing the expected base arm outcomes beforehand, we show that when the expected reward is Lipschitz continuous in the expected base arm outcomes CTS achieves $O(\sum_{i =1}^m\log T/(p_i\Delta_i))$ regret and $O(\max\{\Ex[m\sqrt{T\log T/p^*}],\Ex[m^2/p^*]\})$ Bayesian regret, where $m$ denotes the number of base arms, $p_i$ and $\Delta_i$ denote the minimum non-zero triggering probability and the minimum suboptimality gap of base arm $i$ respectively, $T$ denotes the time horizon, and $p^*$ denotes the overall minimum non-zero triggering probability. We also show that when the expected reward satisfies the triggering probability modulated Lipschitz continuity, CTS achieves $O(\max\{m\sqrt{T\log T},m^2\})$ Bayesian regret, and when triggering probabilities are non-zero for all base arms, CTS achieves $O(1/p^*\log(1/p^*))$ regret independent of the time horizon. Finally, we numerically compare CTS with algorithms based on upper confidence bounds in several networking problems and show that CTS outperforms these algorithms by at least an order of magnitude in majority of the cases.
\end{abstract}

\section{Introduction}
\label{sec:introduction}
How should an advertiser promote its products in a social network to reach to a large set of users with a limited budget \cite{dinh2014cost,tong2017adaptive}? How should a search engine suggest a ranked list of items to its users to maximize the click-through rate \cite{kveton2015cascading}? How should a base station allocate its users to channels to maximize the system throughput \cite{gai2012combinatorial}? How should a mobile crowdsourcing platform dynamically assign available tasks to its workers to maximize the performance \cite{nee2018context}? How can we identify the most reliable paths from source to destination under probabilistic link failures \cite{lee2010diverse}? All of these problems require optimizing decisions among a vast set of alternatives. When the probabilistic description of the environment is fully specified, these problems---and many others---are solved using computationally efficient exact or approximation algorithms. In this paper, we focus on a much more difficult and realistic problem: How should we learn the optimal decisions in these complex problems via repeated interaction with the environment when the probabilistic description of the environment is unknown or only partially known? 

It is natural to assume that the environment is unknown in many real-world applications. For instance, the advertiser may not know with what probability user $i$ will influence its neighbor $j$ in a social network or the search engine may not know with what probability user $i$ will click the item shown on position $j$ beforehand. Moreover, decisions are need to be made sequentially over time. For instance, the recommender system should show a new list of items to each arriving user and the base station should reallocate network resources when the channel conditions change or the users leave/enter the system. Obviously, future decisions of the learner must be guided based on what it has observed thus far, i.e., the trajectory of actions, observations and rewards generated by the learner's past decisions. Importantly, both the cumulative reward of the learner and what it has learned so far also depend on this trajectory. Therefore, the learner needs to balance how much it earns (by exploiting the actions it believes to be the best) and how much it learns (by exploring actions it does not know much about) in order to maximize its long-term performance. In this paper, we solve the formidable task of combinatorial optimization in unknown environments by modeling it as a combinatorial multi-armed bandit (MAB). 

MAB problems have a long history as they exhibit the prime example of the tradeoff between exploration and exploitation \cite{gai2012combinatorial,robbins1952some}. In the classical MAB, at each round the learner selects an arm (action) which yields a random reward that comes from an unknown distribution. The goal of the learner is to maximize its expected cumulative reward over all rounds by learning to select arms that yield high rewards. The learner's performance is measured by its regret with respect to an oracle that always selects the arm with the highest expected reward. It is shown that when the arms' rewards are independent, any uniformly good policy will incur at least logarithmic in time regret \cite{lai1985asymptotically}. 

Several classes of policies are proposed for the learner to minimize its regret. One example is Thompson sampling \cite{thompson1933likelihood,agrawal2012analysis,russo2014learning}, which is a Bayesian method. In this method, the learner keeps a posterior distribution over the expected arm rewards, and at each round takes a sample from each arm's posterior, and then, plays the arm with the largest sample. Reward observed from the played arm is then used to update its posterior. This sampling strategy allows the learner to frequently select the arms whose probabilities of being optimal are the highest based on their posteriors and to occasionally explore inferior arms to refine their posteriors. Policies in the other end of the spectrum use the principle of optimism under the face of uncertainty. Notable examples include policies based on upper confidence bound (UCB) indices \cite{lai1985asymptotically,agrawal1995sample,auer2002finite}, which are usually composed of sample mean reward of an arm plus an exploration bonus that accounts for the uncertainty in the arm's reward estimates. The strategy is to play the arm with the highest UCB index to tradeoff exploration and exploitation. Unlike Thompson sampling, performance of this type of policies heavily rely on the confidence sets used to compute the exploration bonus \cite{russo2014learning}. This together with the superior performance of Thompson sampling documented in numerous applications \cite{chapelle2011empirical,scott2010modern} motivate us to consider a Thompson sampling based approach for our problem. 

Our main focus in this paper, i.e., combinatorial MAB (CMAB) \cite{gai2012combinatorial,cesa2012combinatorial,kveton2015tight,chen2013combinatorial}, is an extension of MAB where the learner selects a \textit{super arm} at each round, which is defined to be a subset of the \textit{base arms}. Then, the learner observes and collects the reward associated with the selected super arm, and also observes the outcomes of the base arms that are in the selected super arm. This type of feedback is also called \textit{semi-bandit} feedback. For instance, when allocating users to orthogonal channels, each user-channel pair represents a base arm, the super arm is the set of user-channel pairs in the selected allocation, outcomes of base arms are indicators of successful packet transmissions and the reward is the number of packets successfully transmitted, i.e., sum of the indicators. While CMAB is general enough to model the aforementioned resource allocation problem, it does not fully capture the probabilistic structure of influence maximization, item list recommendation and reliable packet routing applications discussed in the preceding paragraphs. Therefore, we consider a generalized version of CMAB, called CMAB with probabilistically triggered arms (CMAB-PTA) \cite{chen2016combinatorial}, where the selected super arm probabilistically triggers a set of base arms, and the expected reward obtained in a round is a function of the set of triggered base arms and their expected outcomes. For instance, in influence maximization, each edge of the graph represents a base arm, the super arm is the selected seed set of nodes, outcomes of base arms are indicators of influence propagation on the corresponding edge (see, e.g., the independent cascade model \cite{kempe2003maximizing}) and the reward is the number of influenced nodes, i.e., the set of nodes reachable from the seed set of nodes after the outcomes of base arms are realized. Triggered base arms in this case correspond to the set of edges that originate from all influenced nodes (including the seed set).

The regret for CMAB-PTA is defined as the difference between the expected cumulative reward of an oracle that always selects the super arm with the highest expected reward and that of the learner given a particular environment. Then, the Bayesian regret is the expected regret over all possible environments. Our goal is to design an algorithm that achieves the smallest rate of growth of the (Bayesian) regret over time, as this will ensure that the average reward of the learner will converge to the highest possible expected reward. To this end, we propose a Bayesian algorithm called combinatorial Thompson sampling (CTS) and analyze its regret assuming that the learner does not know the expected base arm outcomes beforehand but has access to an exact optimization oracle. Essentially, this oracle outputs an estimated optimal super arm given estimates of expected base arm outcomes as inputs. \reva{When the expected reward is Lipschitz continuous in the expected base arm outcomes, we show that CTS achieves $O(\sum_{i =1}^m\log T/(p_i\Delta_i))$ regret and $O(\max\{\Ex[m\sqrt{T\log T/p^*}],\Ex[m^2/p^*]\})$ Bayesian regret, where $m$ denotes the number of base arms, $p_i$ denotes the minimum non-zero triggering probability of base arm $i$, $\Delta_i$ denotes the minimum suboptimality gap of base arm $i$, $T$ denotes the time horizon, and $p^*$ denotes the overall minimum non-zero triggering probability. We also show that when the expected reward satisfies the triggering probability modulated (TPM) Lipschitz continuity in \cite{wang2017improving}, which is a stronger assumption than the regular Lipschitz continuity yet still satisfied by the network optimization problems that we consider, CTS achieves $O(\max\{m\sqrt{T\log T},m^2\})$ Bayesian regret independent of the triggering probabilities.}

In addition to these more general cases, we also prove that when triggering probabilities are non-zero for all base arms, CTS achieves $O(1/p^*\log(1/p^*))$ regret independent of the time horizon. This setting is of particular interest since it can model random behavior of users in a recommender system. For instance, a user may rate an item even when it is not in the list of recommended items as a result of an exogenous event (by rating the item on a partner website or by explicitly navigating to the item to rate it). Moreover, it is also closely linked to related work on online learning with probabilistic graph feedback \cite{liu2018information,li2019stochastic} and MAB with side observations \cite{degenne2018bandits}. Specifically, the models in \cite{li2019stochastic} and \cite{degenne2018bandits} become special cases of our work when the graph is fully-connected for the one-step case and connected for the cascade case in \cite{li2019stochastic} and when the probability of having an observation from any arm is non-zero in \cite{degenne2018bandits}.

We complement our theoretical findings via extensive simulations in the following combinatorial network optimization problems: cascading bandits \cite{kveton2015cascading}, probabilistic maximum coverage bandits \cite{chen2016combinatorial} and influence maximization bandits \cite{chen2016combinatorial}. For cascading bandits, we show that CTS, which uses Beta posterior on base arms significantly outperforms all competitor algorithms that use either UCB indices \cite{kveton2015cascading} or Thompson sampling with Gaussian posterior \cite{cheung2019thompson}. The latter finding emphasizes the importance of working with the correct type of posterior. For probabilistic maximum coverage bandits, we show that CTS achieves an order of magnitude improvement over combinatorial UCB (CUCB) in \cite{chen2016combinatorial} when both algorithms use an exact oracle. For influence maximization bandits, we show a similar result even when both algorithms use an approximation oracle instead of an exact oracle. 

In summary, the main contribution of this paper is to analyze Thompson sampling for a very general combinatorial online learning framework that is comprehensive enough to model many different sequential decision-making applications defined over networks and show its optimality both theoretically and experimentally. The rest of the paper is organized as follows. Related work is given in Section~\ref{sec:related} followed by problem formulation in Section~\ref{sec:problem}. Applications of CMAB-PTA are detailed in Section~\ref{sec:networking}. Description of CTS and regret bounds are given in Section \ref{sec:algorithm}. Proofs of the main results are explained in Sections~\ref{sec:proof-bayesian} and \ref{sec:proof-bayesiantpm} (some proofs are left to the supplemental document). Numerical results are presented in Section~\ref{sec:numerical} and concluding remarks are given in Section~\ref{sec:conclusion}.

\section{Related Work}
\label{sec:related}
CMAB has been studied under various assumptions on the relation between super arms, base arms and rewards \cite{cesa2012combinatorial}. Here, we mainly discuss the related works that assume semi-bandit feedback as we do in our work. A version of CMAB in which the expected reward of a super arm is a linear combination of the expected outcomes of the base arms in that super arm is studied in \cite{gai2012combinatorial}. For this problem, it is shown in \cite{kveton2015tight} that a combinatorial version of UCB1 in \cite{auer2002finite} achieves $O(Km\log T/\Delta)$ gap-dependent and $O(\sqrt{KmT\log T})$ gap-free (worst-case) regrets, where $m$ is the number of base arms, $K$ is the maximum number of base arms in a super arm, and $\Delta$ is the gap between the expected reward of the optimal super arm and the second best super arm.

 Later on, this setting is generalized to allow the expected reward of each super arm to be a more general function of the expected outcomes of the base arms that obeys certain monotonicity and bounded smoothness conditions \cite{chen2013combinatorial}. The main challenge in the general case is that the optimization problem itself is NP-hard, but an approximately optimal solution can usually be computed efficiently for many special cases \cite{nemhauser1978analysis}. Therefore, it is assumed that the learner has access to an approximation oracle, which can output a super arm that has expected reward that is at least $\alpha$ fraction of the optimal reward with probability at least $\beta$ when given the expected outcomes of the base arms. Thus, the regret is measured with respect to the $\alpha \beta$ fraction of the optimal reward, and it is proven that a combinatorial variant of UCB1, called CUCB, achieves $O(\sum_{i=1}^m\log T/\Delta_{i})$ regret when the bounded smoothness function is $f(x)=\gamma x$ for some $\gamma>0$, where $\Delta_{i}$ is the minimum gap between the expected reward of the optimal super arm and the expected reward of any suboptimal super arm that contains base arm $i$.

Recently, it is shown in \cite{wang2018thompson} that Thompson sampling can achieve $O(\sum_{i=1}^m\log T/\Delta_{i})$ regret for the general CMAB under a Lipschitz continuity assumption on the expected reward, given that the learner has access to an exact computation oracle, which outputs an optimal super arm when given the set of expected base arm outcomes. Moreover, it is also shown that in general the learner cannot guarantee sublinear regret when it only has access to an approximation oracle. Since the setting studied in this paper is a special case of ours, for our theoretical analysis we also assume that the learner uses an exact computation oracle. Nevertheless, we show in Section~\ref{sec:numerical} that in practice CTS works well even when used with an approximation oracle. Another related work on CMAB \cite{merlis2019batch} considers a new smoothness condition termed the Gini-weighted smoothness on the expected reward. For some problem types, this leads to regret bounds with better dependency on the sizes of super arms when compared with the common linear dependency of the existing algorithms.

Different from CMAB, papers on CMAB-PTA assume that the expected reward is a function of the expected outcomes of the triggered base arms, which is a random superset of base arms in the selected super arm. For this problem, it is shown in \cite{chen2016combinatorial} that logarithmic regret is achievable when the expected reward function has the $\ell_\infty$ bounded smoothness property. However, this bound depends on $1/p^*$, where $p^*$ is the minimum non-zero triggering probability. Later, it is shown in \cite{wang2017improving} that under a stricter smoothness assumption on the expected reward function, called triggering probability modulated (TPM) bounded smoothness, it is possible to achieve regret that does not depend on $1/p^*$. It is also shown in this work that the dependence on $1/p^*$ is unavoidable for the general case. In another work \cite{saritac2017combinatorial}, CMAB-PTA is considered for the case when the arm triggering probabilities are all positive, and it is shown that both CUCB and CTS achieve bounded regret. However, their $O((1/p^*)^4)$ bound has a much worse dependence on $p^*$ than our $O(1/p^*\log(1/p^*))$ bound.

\begin{table}
    \caption{Summary of the related work in comparison with our work.}
    \label{tbl:related}
    \setlength{\tabcolsep}{3pt}
    \begin{tabularx}{\linewidth}{*4{l}X}
        \toprule
        \textbf{Publ.} & \textbf{Algorithm} & \textbf{Oracle} & \textbf{PTAs} & \textbf{Regret Bound} \\
        \midrule
        \cite{chen2013combinatorial} & CUCB & Approx. & No & $O(\sum_i\log T/\Delta_i)$ \\
        \cite{chen2016combinatorial} & CUCB & Approx. & Yes & $O(\sum_i \log T/(p_i\Delta_i))$ \\
        \cite{wang2017improving} & CUCB & Approx. & Yes & $O(\sum_i\log T/\Delta_i)$$^\dagger$ \\
        \cite{wang2018thompson} & CTS & Exact & No & $O(\sum_i \log T/\Delta_i)$ \\
        \cite{saritac2017combinatorial} & CUCB \& CTS & Approx. & Yes$^*$ & $O((1/p^*)^4)$ \\
        \midrule
        \textbf{Ours} & CTS & Exact & Yes & $O(\sum_i \log T/(p_i\Delta_i))$ \\
        & & & Yes & $O(\max\{\Ex[m\sqrt{T\log T/p^*}]$ \\
        & & & & \multicolumn{1}{r}{$,\Ex[m^2/p^*]\})$$^\ddagger$} \\
        & & & Yes & $O(\max\{m\sqrt{T\log T},m^2\})$$^{\dagger\ddagger}$ \\
        & & & Yes$^*$ & $O(1/p^*\log (1/p^*))$ \\
        \bottomrule
    \end{tabularx}
    \smallskip \\
    $^*$The case when the arm triggering probabilities are all positive. \\
    $^\dagger$Under the TPM bounded smoothness assumption. \\
    $^\ddagger$Bound for the Bayesian regret.
\end{table}

Apart from the works mentioned above, numerous other works also tackle related online learning problems. For instance, \cite{kveton2014matroid} considers matroid bandits, which is a special case of CMAB where the super arms are given as independent sets of a matroid with base arms being the elements of the ground set, and the expected reward of a super arm is the sum of the expected outcomes of the base arms in the super arm. Another example is cascading bandits \cite{kveton2015cascading}, which is a special case of CMAB-PTA, where each super arm corresponds to a ranked list of items and base arms are triggered according to a user click model. A plethora of papers exist on UCB based policies for variants of these two models (see e.g., \cite{talebi2016optimal} for a variant of matroid bandits and \cite{kveton2015combinatorial} and \cite{li2016contextual} for variants of cascading bandits.) Apart from these, \cite{cheung2019thompson} considers Thompson sampling with Gaussian posterior for cascading bandits and proves that the worst-case regret is $\tilde{O}(\sqrt{KmT})$. We show in Section~\ref{sec:numerical} that CTS significantly outperforms their algorithm for cascading bandits. We think that this is the case in practice because Beta posterior is more suitable in modeling click probabilities compared to Gaussian posterior.

Several other works focus on contextual CMAB \cite{qin2014contextual,li2016contextual,saritac2018online}, CMAB with adversarial rewards \cite{audibert2013regret,combes2015combinatorial} and CMAB with knapsacks \cite{sankararaman2018combinatorial}. Most recently there has been a surge of interest in analyzing CMAB under the full-bandit feedback setting, where the learner only observes the reward of the selected super arm but not the outcomes of the base arms \cite{agarwal2018regret,rejwan2019combinatorial}. For instance, \cite{rejwan2019combinatorial} uses a sampling method based on Hadamard matrices to estimate base arm rewards from full-bandit feedback. On the other hand, \cite{lin2014combinatorial} considers a more general feedback model where the learner observes a linear combination of base arm's rewards. Table~\ref{tbl:related} compares our work with the most closely related publications in terms of their assumptions and the regret bounds they show.

\section{Problem Formulation}
\label{sec:problem}
CMAB-PTA is a decision-making problem where the learner interacts with its environment through $m$ base arms, indexed by the set $[m]\coloneqq\{1,2,...,m\}$ sequentially over rounds indexed by $t\in[T]$. In this paper, we consider the model introduced in \cite{chen2016combinatorial} and borrow the notation from \cite{wang2018thompson}. In this model, the following events take place in order in each round~$t$:
\begin{itemize}
    \item The learner selects a subset of base arms, denoted by $S(t)$, which is called a super arm.
    \item $S(t)$ causes some other base arms to probabilistically trigger based on a stochastic triggering process, which results in a set of triggered base arms $S'(t)$ that contains $S(t)$.
    \item The learner obtains a reward that depends on $S'(t)$ and observes the outcomes of the base arms in $S'(t)$.
\end{itemize}

Next, we describe in detail the base arm outcomes, the super arms, the triggering process, the reward, the observation (feedback) model and the regret. 

\subsection{Base Arm Outcomes}

In each round $t$, the environment draws a random outcome vector $\bm{X}(t)\coloneqq(X_1(t),X_2(t),\ldots,X_m(t))$ from a probability distribution $D$ on $[0,1]^m$ independent of the previous rounds, where $X_i(t)$ represents the outcome of base arm $i$. $D$ is unknown by the learner, but it belongs to a class of distributions ${\cal D}$ which is known by the learner. We define the mean outcome (parameter) vector as $\bm{\mu}\coloneqq(\mu_1,\mu_2,\ldots,\mu_m)$, where $\mu_i\coloneqq\Ex_{\bm{X}\sim D}[X_i(t)]$, and use $\bm{\mu}_S$ to denote the projection of $\bm{\mu}$ on $S$ for $S\subseteq[m]$.

Since CTS computes a posterior over $\bm{\mu}$, the following assumption is made to have an efficient and simple update of the posterior distribution. 

\begin{assumption} \label{asm:independence}
	The outcomes of all base arms are mutually independent, i.e., $D=D_1\times D_2\times\cdots\times D_m$.
\end{assumption}

Note that this independence assumption holds in many applications, including the influence maximization problem with independent cascade influence propagation model \cite{kempe2003maximizing}.

\subsection{Super Arms and the Triggering Process}

The learner is allowed to select $S(t)$ from a subset of $2^{[m]}$ denoted by $\mathcal{I}$, which corresponds to the set of feasible super arms. Once $S(t)$ is selected, all base arms $i\in S(t)$ are immediately triggered. These arms can trigger other base arms that are not in $S(t)$, and those arms can further trigger other base arms, and so on. At the end, a random superset $S'(t)$ of $S(t)$ is formed that consists of all triggered base arms as a result of selecting $S(t)$. We have $S'(t) \sim D^{\mathrm{trig}}(S(t), \bm{X}(t))$, where $D^{\mathrm{trig}}$ is the probabilistic triggering function that describes the triggering process. For instance, in the influence maximization problem, $D^{\mathrm{trig}}$ may correspond to the independent cascade influence propagation model defined over a given influence graph \cite{kempe2003maximizing}. The triggering process can also be described by a set of triggering probabilities. For each $i \in [m]$ and $S \in \mathcal{I}$, $p_i^{D',S}$ denotes the probability that base arm $i$ is triggered when super arm $S$ is selected given that the arm outcome distribution is $D'\in\mathcal{D}$. For simplicity, we let $p_i^S=p_i^{D,S}$, where $D$ is the true arm outcome distribution. Let $\tilde{S}\coloneqq\{i\in[m]:p_i^S>0\}$ be the set of all base arms that could potentially be triggered by super arm $S$, which is called the \textit{triggering set} of $S$. We have that $S(t)\subseteq S'(t)\subseteq \tilde{S}(t)\subseteq[m]$. We define $p_i\coloneqq\min_{S\in\mathcal{I}:i\in\tilde{S}}p_i^S$ as the minimum nonzero triggering probability of base arm $i$, and $p^*\coloneqq\min_{i\in[m]}p_i$ as the minimum nonzero triggering probability.

\reva{Before moving on, we would like to point out that the entire triggering process could have been represented by writing $S'(t)\sim \bar{D}^{\mathrm{trig}}(S(t))$, where any possible dependence of the process on the outcome distribution $D$ would have been hidden inside $\bar{D}^{\mathrm{trig}}$. Instead, we chose to break down the triggering process into two stages: $\bm{X}(t)\sim D$ and $S'(t)\sim D^{\mathrm{trig}}(S(t),\bm{X}(t))$, where $D$ and $D^{\mathrm{trig}}$ together are equivalent to $\bar{D}^{\mathrm{trig}}$. This is motivated by the prior knowledge of the learner. Note that, while the learner fully knows $D^{\mathrm{trig}}$, it does not know anything about $D$ except the class of distributions $\mathcal{D}$ that it belongs to, resulting in only a partial knowledge about $\bar{D}^{\mathrm{trig}}$.}

\subsection{Reward}

At the end of round $t$, the learner receives a reward that depends on the set of triggered arms $S'(t)$ and the outcome vector $\bm{X}(t)$, which is denoted by $R(S'(t),\bm{X}(t))$. For simplicity of notation, we also use $R(t)=R(S'(t),\bm{X}(t))$ to denote the reward in round $t$. Note that whether a base arm is in the selected super arm or is triggered afterwards is not relevant in terms of the reward. We assume that the expected reward depends on the mean outcome vector in a specific way by making the following mild assumptions about the expected reward function. We note that these assumptions are standard in the CMAB literature \cite{chen2016combinatorial,wang2018thompson} and hold for the networking applications given in Section~\ref{sec:networking}. The first assumption states that the expected reward is only a function of $S(t)$ and $\bm{\mu}$.

\begin{assumption} \label{asm:rdependence}
    The expected reward of super arm $S\in\mathcal{I}$ only depends on $S$ and the mean outcome vector $\bm{\mu}$, i.e., there exists a function $r$ such that 
    \begin{align*}
        \Ex[R(t)] &= \Ex_{S'(t)\sim D^{\mathrm{trig}}(S(t),\bm{X}(t)),\bm{X}(t)\sim D}[R(S'(t),\bm{X}(t))] \\
        &= r(S(t),\bm{\mu}) ~.
    \end{align*}
\end{assumption}

In order to learn the best action, we require the estimate of the expected reward vector to converge to the true expected reward vector as the number of observations increases. This can be done when the expected reward varies smoothly with the mean outcome vector. Below, we state a form of continuity for the expected reward.

\begin{assumption}{(Lipschitz continuity)} \label{asm:rsmoothness}
    There exists a constant $B > 0$, such that for every super arm $S$ and every pair of mean outcome vectors $\bm{\mu}$ and $\bm{\mu'}$, we have 
    \begin{align*}
        |r(S,\bm{\mu})-r(S,\bm{\mu'})| \leq B\|\bm{\mu}_{\tilde{S}}-\bm{\mu'}_{\tilde{S}}\|_1
    \end{align*}
    where $\|\cdot\|_1$ denotes the $l_1$ norm.
\end{assumption}

\reva{In addition to Lipschitz continuity, we also consider the \textit{triggering probability modulated} (TPM) Lipschitz continuity introduced in \cite{wang2017improving}. This is a stricter assumption than the regular Lipschitz continuity (one implies the other) but leads to tighter regret bounds in terms of the triggering probabilities. All of the networking applications considered in Section~\ref{sec:networking} still satisfy the TPM Lipschitz continuity.}

\reva{\begin{assumption}{(Triggering probability modulated Lipschitz continuity)} \label{asm:rsmoothnesstpm}
There exists a constant $B'>0$, such that for every super arm $S$ and every pair of outcome distributions $D$ and $D'$ with mean outcome vectors $\bm{\mu}$ and $\bm{\mu'}$ respectively, we have
    \begin{align*}
        |r(S,\bm{\mu})-r(S,\bm{\mu'})| \leq B'\sum_{i\in\tilde{S}}p_i^{D,S}|\mu_i-\mu'_i| ~.
    \end{align*}
\end{assumption}}

\reva{Finally, we require a monotonicity assumption in order to facilitate the UCB-based analysis that some of our results rely on, namely Theorems~\ref{thm:bayesian} and \ref{thm:bayesiantpm}. Again, all of the networking applications considered in Section~\ref{sec:networking} satisfy the following monotonicity assumption.}

\reva{\begin{assumption} \label{asm:rmonotonicity}
    For every super arm $S$ and every pair of mean outcome vectors $\bm{\mu}$ and $\bm{\mu'}$, we have $r(S,\bm{\mu})\leq r(S,\bm{\mu'})$ if $\mu_i\leq\mu'_i$ for all $i\in[m]$.
\end{assumption}}

\subsection{Observation Model}

We consider the semi-bandit feedback model, where at the end of round $t$, the learner observes the individual outcomes of the triggered arms, denoted by $Q(S'(t),\bm{X}(t))\coloneqq\{(i,X_i(t)):i\in S'(t)\}$. Again, for simplicity of notation, we also use $Q(t)=Q(S'(t),\bm{X}(t))$ to denote the observation at the end of round $t$. Based on this, the only information available to the learner when choosing the super arm to select in round $t+1$ is its observation history, given as $\mathcal{F}_t\coloneqq\{(S(\tau),Q(\tau)):\tau\in[t]\}$. 

In short, the tuple $([m],\mathcal{I},D,D^{\mathrm{trig}},R)$ constitutes a CMAB-PTA problem instance. Among the elements of this tuple only $D$ is unknown to the learner.

\subsection{Regret} \label{sec:problem-regret}

In order to evaluate the performance of the learner, we define the set of optimal super arms given an $m$-dimensional parameter vector $\bm{\theta}$ as $\mathrm{OPT}(\bm{\theta})\coloneqq\argmax_{S\in\mathcal{I}}r(S,\bm{\theta})$. We use $\mathrm{OPT}\coloneqq\mathrm{OPT}(\bm{\mu})$ to denote the set of optimal super arms given the true mean outcome vector $\bm{\mu}$. Based on this, we let $S^*$ to represent a specific super arm in $\argmin_{S\in\mathrm{OPT}}|\tilde{S}|$, which is the set of super arms that have triggering sets with minimum cardinality among all optimal super arms. We also let $k^*\coloneqq|S^*|$ and $\tilde{k}^*\coloneqq|\tilde{S}^*|$.

Next, we define the suboptimality gap due to selecting super arm $S\in\mathcal{I}$ as $\Delta_S\coloneqq r(S^*,\bm{\mu})-r(S,\bm{\mu})$, the maximum suboptimality gap as $\Delta_{\max}\coloneqq\max_{S\in\mathcal{I}}\Delta_S$, and the minimum suboptimality gap of base arm $i$ as $\Delta_i\coloneqq\min_{S\in\mathcal{I}-\mathrm{OPT}:i\in\tilde{S}}\Delta_S$.\footnote{If there is no such super arm $S$, let $\Delta_i=\infty$.} The goal of the learner is to minimize the (expected) regret over the time horizon $T$, given by
\begin{align}
    \mathrm{Reg}(T) &\coloneqq \Ex\left[ \sum_{t=1}^T(r(S^*,\bm{\mu})-r(S(t),\bm{\mu})) \middle| \bm{\mu} \right] \nonumber \\
    &= \Ex\left[ \sum_{t=1}^T\Delta_{S(t)} \middle| \bm{\mu} \right] ~. \label{eqn:regret}
\end{align}

\reva{In addition to the expected regret, we also consider the \textit{Bayesian regret}, given by
\begin{align*}
    \mathrm{BayReg}(T) &\coloneqq \Ex\left[ \sum_{t=1}^T(r(S^*,\bm{\mu})-r(S(t),\bm{\mu})) \right] \\
    &= \Ex_{\bm{\mu}}[\mathrm{Reg}(T)]
\end{align*}
where the true mean outcome vector $\bm{\mu}$ is viewed as a random variable. For simplicity, we will assume that $\bm{\mu}$ has a uniform prior. However, this can easily be extended to any other Dirichlet prior simply by modifying the initial values of $a_i$'s and $b_i$'s in Algorithm~\ref{alg:CTS}, which determine the initial prior over the base arm outcomes. It is important to note here that asymptotic bounds on the Bayesian regret are essentially asymptotic (gap-free) bounds on the regret \cite{russo2014learning}. Formally, if $\mathrm{BayReg}(T)\in O(f(T))$ for some non-negative function $f(T)$, then $\mathrm{Reg}(T)\in O_P(f(T))$, that is there exists $T_0>0$ such that for all $\epsilon>0$ there exists $M>0$ such that $\Pr(\mathrm{Reg}(T)/f(T)\geq M)\leq\varepsilon$ for all $T>T_0$.}

\section{Networking Applications}
\label{sec:networking}
Here, we introduce three networking applications of CMAB-PTA: cascading bandits, probabilistic maximum coverage bandits, and influence maximization bandits. Numerical experiments given in Section~\ref{sec:numerical} explore specific cases of all these problems that are generated either synthetically or from real-world data.

\subsection{Cascading Bandits} \label{sec:networking-cascading}

\subsubsection{Disjunctive Form for Search Engine Optimization}
In the disjunctive form of the cascading bandit problem \cite{kveton2015cascading}, a search engine outputs a list of $K$ web pages for each of its \reva{$W$ users} among a set of $V$ web pages. Then, the users examine their respective lists, and click on the first page that they find attractive. If all pages fail to attract them, they do not click on any page. The goal of the search engine is to maximize the number of clicks.

This problem can be modeled as an instance of CMAB-PTA as follows. The base arms are page-user pairs $(i,j)$, where $i\in[V]$ and $j\in[W]$. User $j$ finds page $i$ attractive independent of other users and other pages with probability $p_{i,j}$. The super arms are $W$-many lists of $K$-tuples, where each $K$-tuple represents the list of pages shown to a user. Given a super arm $S$, let $S(k,j)$ denote the $k$th page that is selected for user $j$. Then, the triggering probabilities can be written as
\begin{align*}
    p_{(i,j)}^S &= \begin{cases}
        1 & \text{if } i=S(1,j) \\
        \prod_{k'=1}^{k-1}(1-p_{S(k',j),j}) & \text{if } \exists k\neq 1:i=S(k,j) \\
        0 & \text{otherwise} ~,
    \end{cases}
\end{align*}
that is we observe feedback for a top selection immediately, and observe feedback for the other selections only if all previous selections fail to attract the user. The expected reward of playing super arm $S$ can be written as
\begin{align*}
    r(S,\bm{p}) = \sum_{j=1}^W\left( 1-\prod_{k=1}^K(1-p_{S(k,j),j}) \right)
\end{align*}
for which Assumptions~\ref{asm:rsmoothness} and \ref{asm:rsmoothnesstpm} hold when $B=1$ and $B'=1$ respectively.

\subsubsection{Conjunctive Form for Network Routing Reliability}
One can also consider the conjunctive analogue of the problem, where the goal of the search engine is to---somewhat peculiarly---maximize the number of users with lists that do not contain any unattractive page, and when examining their lists, users provide feedback by reporting the first unattractive page. Formally,
\begin{align*}
	p_{(i,j)}^S &= \begin{cases}
		1 & \text{if } i=S(1,j) \\
		\prod_{k'=1}^{k-1}p_{S(k',j),j} & \text{if } \exists k\neq 1:i=S(k,j)  \\
		0 & \text{otherwise}
	\end{cases}
\end{align*}
and
\begin{align*}
	r(S,\bm{p}) = \sum_{j=1}^W\prod_{k=1}^K p_{S(k,j),j} ~. 
\end{align*}

This conjunctive form fits particularly well to the network reliability problem \cite{lee2010diverse}, where we are interested in finding the most reliable routing path in a communication network. We consider routing paths as super arms, $\mathcal{I}$ being the set of all possible routing paths. Each routing path $S\in\mathcal{I}$ consists of a variable number of ordered links that correspond to the base arms. We denote the index of $k$th link in routing path $S$ as $S(k)$ and the length of the path as $|S|$. Each link $i\in[m]$ in a routing path can fail independently from all other links with probability $1-p_i$. Then, the probabilistic reliability of a routing path is defined as the probability of successful operation with no link in the path failing.

Since we can only observe whether a link has failed or not up to the first link that has failed, the triggering probability of link $i$ when routing path $S$ is selected can be written as
\begin{align*}
    p_i^S &= \begin{cases}
        1 & \text{if } i=S(1) \\
        \prod_{k'=1}^{k-1}p_{S(k')} & \text{if } \exists k\neq 1:i=S(k) \\
        0 & \text{otherwise}
    \end{cases}
\end{align*}
and the probabilistic reliability of routing path $S$---in other words, the expected reward---becomes
\begin{align*}
    r(S,\bm{p}) = \prod_{k=1}^{|S|}p_{S(k)} ~.
\end{align*}

\subsection{Probabilistic Maximum Coverage Bandits}

In the probabilistic maximum coverage problem, an online shopping site advertises $K$ items that are selected from a catalog of $V$ items to its \reva{$W$ users}. Each user inspects all of the items that are advertised and likes one of the attractive items. The users do not like any item if none of the items attract them. The goal of the shopping site is to maximize the number of likes. Analogous to cascading bandits, in this problem, base arms are item-user pairs $(i,j)$, where $i\in[V]$ and $j\in[W]$. User $j$ finds item $i$ attractive independent of other users and other items with probability $p_{i,j}$. The super arms are the set of all pairs $(i,j)$ such that item $i$ is the element of a size-$K$ subset of $[V]$.

This can also model the problem of allocating orthogonal channels to secondary users in a cognitive radio network \cite{gai2012combinatorial}. Consider $V$ as the number of orthogonal channels, $W$ as the number of secondary users ($V>W$), and $p_{i,j}$ as the expected throughput that user $j$ can obtain using channel $i$. We would like to maximize the expected sum throughput by allocating each user $j$ a unique channel $c_j\in[V]$ so that $c_j=c_{j'}$ if and only if $j=j'$ for all $j,j'\in[W]$. Given one such allocation, the corresponding super arm would be the set $S=\{(c_j,j)\}_{j=1}^W$ and the expected reward of it can be written as $r(S,\bm{p})=\sum_{(i,j)\in S} p_{i,j}$. \reva{Allocating orthogonal channels to secondary users can also be conceptualized as allocating tasks to workers in a mobile crowdsourcing platform \cite{nee2018context,nika2020contextual}. Then, $p_{i,j}$ would be the probability of worker $j$ completing task $i$ successfully and $r(S,\bm{p})$ would be the expected number of completed tasks.}

In its classical form, this problem does not have any PTAs. In order to provide an example case with strictly positive triggering probabilities, we introduce the \textit{word-of-mouth effect} as follows. Regardless of the shopping site's decisions, we assume that users inspect, i.e., they explicitly search or navigate to, unadvertised items independently with probability $p^*$.\footnote{For simplicity we assume that $p^*$ is the same for all items while it can be different in practice.} This can happen if users hear about the items outside of the shopping site (e.g., from their friends or from another venue). Then, the triggering probabilities can be written as
\begin{align*}
    p_{(i,j)}^S &=  \begin{cases}
        1 & \text{if } (i,j)\in S \\
        p^* & \text{otherwise}
    \end{cases}
\end{align*}
and the expected reward of super arm $S$ can be written as
\begin{align*}
    r(S,\bm{p}) = \sum_{j=1}^W\left( 1-\prod_{i=1}^V(1-p_{(i,j)}^Sp_{i,j}) \right)
\end{align*}
for which Assumption~\ref{asm:rsmoothness} and \ref{asm:rsmoothnesstpm} hold when $B=1$ and $B'=1$ respectively.

\subsection{Influence Maximization Bandits}

In the influence maximization problem with the independent cascade model \cite{kempe2003maximizing}, the learner is given a directed graph denoted by $G=(V,E)$, where $V$ is the set of nodes and $E$ is the set of edges. The learner selects and triggers a set of nodes $S\subseteq V$ such that $|S|=K$, where $K$ is one of the problem parameters. This is the first iteration of a diffusion process. In each subsequent iteration, a node $i$ that was triggered in the previous iteration might trigger another node $j$ that is not triggered yet if $j$ is adjacent to one of its outgoing edges. This happens with probability $p_{i,j}$ independently from the states of all other nodes. The diffusion process ends when no new node triggers in an iteration. The goal of the learner is to maximize---through the initial decision of nodes---the number of triggered nodes at the end of the diffusion process.

The problem can be modeled as a CMAB problem with PTAs, where base arms are edges $(i,j)\in E$ and super arms are the set of all edges $(i,j)$ such that $i\in S$.\footnote{This is equivalent to defining the super arm as $S$ itself.} Assumption~\ref{asm:rsmoothness} holds as proven in Lemma~6 in \cite{chen2016combinatorial} and Assumption~\ref{asm:rsmoothnesstpm} holds as proven in Lemma~2 in \cite{wang2017improving}.

\section{Combinatorial Thompson Sampling}
\label{sec:algorithm}
CTS is a Bayesian algorithm that selects super arms by sampling from posterior distributions of base arms. Its pseudocode is given in Algorithm~\ref{alg:CTS}. We assume that the learner has access to an exact computation oracle, which takes as input an $m$-dimensional parameter vector $\bm{\theta}$ and the problem structure $([m],\mathcal{I},D^{\text{trig}},R)$, and outputs a super arm, denoted by $\mathrm{Oracle}(\bm{\theta})$ such that $\mathrm{Oracle}(\bm{\theta})\in\mathrm{OPT}(\bm{\theta})$. CTS keeps a Beta posterior over the mean outcome of each base arm. At the beginning of round $t$, for each base arm $i$ it draws a sample $\theta_i(t)$ from its posterior distribution. Then, it forms the parameter vector in round $t$ as $\bm{\theta}(t)\coloneqq(\theta_1(t),\ldots,\theta_m(t))$, gives it to the exact computational oracle, and selects the super arm $S(t)=\mathrm{Oracle}(\bm{\theta}(t))$. At the end of the round, CTS updates the posterior distributions of the triggered base arms using the observation $Q(t)$. 

\begin{algorithm}
	\caption{Combinatorial Thompson Sampling (CTS)}
	\label{alg:CTS}
	\begin{algorithmic}[1]
		\STATE For each base arm $i$, let $a_i=1$, $b_i=1$
		\FOR{$t=1,2,\ldots$}
		\STATE {For each base arm $i$, draw a sample $\theta_i(t)$ from Beta distribution $\beta(a_i,b_i)$; let $\bm{\theta}(t)\coloneqq(\theta_1(t),\ldots,\theta_m(t))$}
		\STATE {Select super arm $S(t)=\mathrm{Oracle}(\bm{\theta}(t))$, get the observation $Q(t)$}
		\FORALL{$(i,X_i)\in Q(t)$}
		\STATE $Y_i\gets 1$ with probability $X_i$, $0$ with probability $1-X_i$
		\STATE $a_i\gets a_i+Y_i$
		\STATE $b_i\gets b_i+(1-Y_i)$
		\ENDFOR
		\ENDFOR
	\end{algorithmic}
\end{algorithm}

    \subsection{Regret of CTS under Lipchitz Continuity}
    \label{sec:theorem}
    \begin{theorem} \label{thm:main}
    Under Assumptions~\ref{asm:independence}, \ref{asm:rdependence}, and \ref{asm:rsmoothness}, for all $D$, the regret of CTS by round $T$ is bounded as
    \begin{align*}
        \mathrm{Reg}(T) &\leq \sum_{i=1}^m \max_{S\in\mathcal{I}-\mathrm{OPT}:i\in\tilde{S}} \frac{16B^2|\tilde{S}|\log T}{(1-\rho)p_i(\Delta_S-2B(\tilde{k}^{*2}+2)\varepsilon)} \\
        &\hspace{12pt} + \left( 3+\frac{\tilde{K}^2}{(1-\rho)p^*\varepsilon^2}+\frac{2\mathbb{I}\{p^*<1\}}{\rho^2p^*} \right) m\Delta_{\max} \\
        &\hspace{12pt} + \alpha\frac{8\tilde{k}^*}{p^*\varepsilon^2} \left(\frac{4}{\varepsilon^2}+1\right)^{\tilde{k}^*} \log\frac{\tilde{k}^*}{\varepsilon^2} \Delta_{\max}
    \end{align*}
    for all $\rho\in(0,1)$, and for all $\varepsilon\in(0,1/\sqrt{e}]$ such that $\forall S\in\mathcal{I}-\text{OPT}, \Delta_S>2B(\tilde{k}^{*2}+2)\varepsilon$, where $B$ is the Lipschitz constant in Assumption~\ref{asm:rsmoothness}, $\alpha >0$ is a problem independent constant that is also independent of $T$, and $\tilde{K}\coloneqq\max_{S\in\mathcal{I}}|\tilde{S}|$ is the maximum triggering set size among all super arms.
\end{theorem}

We compare the result in Theorem~\ref{thm:main} with \cite{chen2016combinatorial}, which shows that the regret of CUCB is $O(\sum_{i\in[m]}\log T/(p_i\Delta_i))$ given an $\ell_{\infty}$ bounded smoothness condition on the expected reward function and a bounded smoothness function of $f(x)=\gamma x$. When $\varepsilon$ is sufficiently small, the regret bound in Theorem~\ref{thm:main} is asymptotically equivalent to the regret bound for CUCB (in terms of the dependence on $T$, $p_i$, and $\Delta_i$ for $i \in [m]$). For the case with $p^*=1$ (no probabilistic triggering), the regret bound in Theorem~\ref{thm:main} matches with the regret bound in Theorem~1 in \cite{wang2018thompson} (in terms of the dependence on $T$ and $\Delta_i$ for $i \in [m]$).

As final remarks, it is shown in Theorem~3 in \cite{wang2017improving} that the $1/p_i$ factor that multiplies the $\log T$ term is unavoidable in general. \reva{Moreover, regarding the exponential term $(4/\varepsilon^2+1)^{\tilde{k}^*}$, it is shown in Theorem~3 in \cite{wang2018thompson} that there is at least one instance of CMAB (hence, also an instance of CMAB-PTA) where the regret of CTS is $\Omega(2^{k^*})$. Intuitively, such an exponential term is unavoidable since for CTS to select an optimal super arm that can trigger $\tilde{k}^*$ base arms, all of the samples from those $\tilde{k}^*$ base arms should independently be close to their true means.} The proof of Theorem~\ref{thm:main} is given in the supplemental document. It can also be found in the conference version of the paper \cite{huyuk2019analysis}.

    \subsection{Bayesian Regret of CTS under Lipchitz Continuity}
    \label{sec:theorem-bayesian}
    \reva{\begin{theorem} \label{thm:bayesian}
    Under Assumptions~\ref{asm:independence}, \ref{asm:rdependence}, \ref{asm:rsmoothness}, and \ref{asm:rmonotonicity}, when averaged over $D$, the Bayesian regret of CTS by round $T$ is bounded as
    \begin{align*}
        \mathrm{BayReg}(T) &\leq 4mB\sqrt{\frac{T(2+6\log T)}{(1-\rho)}} \Ex_{\bm{\mu}}\left[\sqrt{\frac{1}{p^*}}\right]  \\
        &\hspace{12pt} + 8m^2B\left(1+\frac{1}{\rho^2}\Ex_{\bm{\mu}}\left[\frac{1}{p^*}\right]\right) 
    \end{align*}
    for all $\rho\in(0,1)$, where $B$ is the Lipschitz constant in Assumption~\ref{asm:rsmoothness}.
\end{theorem}}

As mentioned in Section~\ref{sec:problem-regret}, the Bayesian regret bound in Theorem~\ref{thm:bayesian} can be interpreted as a gap-free regret bound for CTS that holds asymptotically.

    \subsection{Bayesian Regret of CTS under the TPM Lipchitz Continuity}
    \label{sec:theorem-bayesiantpm}
    \reva{\begin{theorem} \label{thm:bayesiantpm}
    Under Assumptions~\ref{asm:independence}, \ref{asm:rdependence}, \ref{asm:rsmoothnesstpm}, and \ref{asm:rmonotonicity}, when averaged over $D$, the Bayesian regret of CTS by round $T$ is bounded as
    \begin{align*}
        \mathrm{BayReg}(T) &\leq 16mB'(1+\sqrt{2})\sqrt{(1+4\log T)T} \\
        &\hspace{12pt} +4mB'+8m^2B'
    \end{align*}
    where $B'$ is the Lipschitz constant in Assumption~\ref{asm:rsmoothnesstpm}.
\end{theorem}}

We improve the Bayesian regret bound in Theorem~\ref{thm:bayesian} under the stricter TPM Lipchitz continuity assumption and obtain a regret bound that is completely-free of triggering probabilities. Similar to Theorem~\ref{thm:bayesian}, the Bayesian regret bound in Theorem~\ref{thm:bayesiantpm} can be interpreted as an asymptotic regret bound for CTS.

    \subsection{Regret of CTS for Strictly Positive Triggering Probabilities}
    \label{sec:theorem-positive}
    We improve the regret bound in Theorem~\ref{thm:main} when all triggering probabilities are strictly positive.

\begin{theorem} \label{thm:positive}
    Under Assumptions~\ref{asm:independence}, \ref{asm:rdependence}, and \ref{asm:rsmoothness}, for all $D$ such that $\forall i\in[m],S\in\mathcal{I}, p_i^{D,S}\geq p^*>0$, the regret of CTS by round $T$ is bounded as
   \begin{align*}
        \mathrm{Reg}(T) &\leq \max\left\{ 16mB\sqrt{\frac{e}{(1-\rho)p^*}}, \right. \\
        &\hspace{12pt} \max_{S\in\mathcal{I}-\mathrm{OPT}}\left\{ \frac{128mB^2|\tilde{S}|}{(1-\rho)p^*(\Delta_S-2B(\tilde{k}^{*2}+2)\varepsilon)} \right. \\
        &\hspace{24pt} \times \left.\left. \log\frac{4B|\tilde{S}|}{(1-\rho)p^*(\Delta_S-2B(\tilde{k}^{*2}+2)\varepsilon)} \right\}\right\} \\
        &\hspace{12pt} + \left( 5+\frac{\tilde{K}^2}{(1-\rho)p^*\varepsilon^2}+\frac{2\mathbb{I}\{p^*<1\}}{\rho^2p^*} \right) m\Delta_{\max} \\
        &\hspace{12pt} + \alpha\frac{8\tilde{k}^*}{p^*\varepsilon^2} \left(\frac{4}{\varepsilon^2}+1\right)^{\tilde{k}^*} \log\frac{\tilde{k}^*}{\varepsilon^2} \Delta_{\max}
    \end{align*}
    for all $\rho\in(0,1)$, and for all $\varepsilon\in(0,1/\sqrt{e}]$ such that $\forall S\in\mathcal{I}-\text{OPT},\: \Delta_S>2B(\tilde{k}^{*2}+2)\varepsilon$, where $B$ is the Lipschitz constant in Assumption~\ref{asm:rsmoothness}, $\alpha >0$ is a problem independent constant that is also independent of $T$, and $\tilde{K}\coloneqq\max_{S\in\mathcal{I}}|\tilde{S}|$ is the maximum triggering set size among all super arms.
\end{theorem}

Note that having all triggering probabilities be strictly positive makes the exploration aspect of the MAB problem trivial. No matter which actions the learner takes, all base arms provide occasional feedback. As a result of this, the upper bound for the expected regret becomes independent of the time horizon $T$. We compare the result of Theorem~\ref{thm:positive} with \cite{saritac2017combinatorial}, which shows a similar bound for CTS in the exact same setting. While the bound in \cite{saritac2017combinatorial} is on order $O((1/p^*)^4)$ with respect to $p^*$, the bound in Theorem~\ref{thm:positive} is on order $O(1/p^*\log(1/p^*))$.

As a final remark, we observe that the regret bound in Theorem~\ref{thm:positive} does not match the lower bound on order $\Omega(\log(1/p^*))$ given in Theorem~1 in \cite{degenne2018bandits} proven for a special case of our setting, where rewards only depend on the selected arm. Assumptions~\ref{asm:rsmoothness} and \ref{asm:rsmoothnesstpm}, on the other hand, allow rewards to depend on all arms in the triggering set of the selected super arm either independent of or proportionally to their triggering probabilities. Considering how the reward model in \cite{degenne2018bandits} satisfies both Assumption~\ref{asm:rsmoothness} and Assumption~\ref{asm:rsmoothnesstpm} and how Assumption~\ref{asm:rsmoothnesstpm} is necesary to get rid of the $1/p^*$ terms in the previously discussed upper bounds, showing an upper bound on order $O(\log(1/p^*))$ instead of order $O(1/p^*\log(1/p^*))$ for the case with strictly positive triggering probabilities might only be possible under Assumption~\ref{asm:rsmoothnesstpm}. The proof of Theorem~\ref{thm:positive} is given in the supplemental document.

\section{\reva{Proof of Theorem~\ref{thm:bayesian}}}
\label{sec:proof-bayesian}
We extend the proof technique used in \cite{russo2014learning} to CMAB-PTA. The technique relies on Fact~\ref{fct:decomp}, which establishes a relationship between Thompson sampling and upper confidence sequences commonly encountered in UCB-based analyses. According to Fact~\ref{fct:decomp}, the Bayesian regret is bounded by the difference between the true rewards and an upper confidence bound for the estimated rewards of the selected super arm and the optimal super arm. We show that these differences either shrink quickly as sample size increases (for the selected super arm) or are less than zero (for the optimal super arm) with overwhelming probability.

\subsection{Preliminaries} \label{sec:proof-bayesian-preliminaries}

All equalities and inequalities concerning random variables hold with probability $1$. The complement of set $\mathcal{S}$ is denoted by $\neg\mathcal{S}$. The indicator function is given as $\In\{\cdot\}$. $M_i(t)\coloneqq\sum_{\tau=1}^{t-1} \In\{i\in\tilde{S}(\tau)\}$ denotes the number of times base arm $i$ is tried to be triggered (i.e.\ it was in the triggering set of the selected super arm) until round $t$, $N_i(t)\coloneqq\sum_{\tau=1}^{t-1}\In\{i\in S'(\tau)\}$ denotes the number of times base arm $i$ is triggered until round $t$, and $\hat{\mu}_i(t)\coloneqq\sum_{\tau:\tau<t,i\in S'(\tau)} Y_i(\tau)/N_i(t)$ denotes the empirical mean outcome of base arm $i$ at the start of round $t$, where $Y_i(t)$ is the Bernoulli random variable with mean $X_i(t)$ that is used for updating the posterior distribution that corresponds to base arm $i$ in CTS.

Given a particular base arm $i\in[m]$, let $\tau_w^i$ be the round for which base arm $i$ is in the triggering set $\tilde{S}(t)$ of the selected super arm $S(t)$ for the $w$th time and let $\tau_0^i=0$. Note that we have $i\in\tilde{S}(\tau_{w+1}^i)$ and $M_i(\tau_{w+1}^i)=w$ for all $w\geq 0$. In order to decompose the regret, we make use of an upper confidence bound sequence $U(S,t)\coloneqq r(S,\bm{\bar{\mu}}(t))$ for the reward of super arm $S\in\mathcal{I}$ at round $t$, where $\bm{\bar{\mu}}(t)=(\bar{\mu}_1(t),\ldots,\bar{\mu}_m(t))$ and
\begin{align*}
    \bar{\mu}_i(t) = \hat{\mu}_i(t)+\min\left\{1,\sqrt{\frac{2+6\log T}{N_i(t)}}\right\} ~.
\end{align*}
We also make use of the following events:
\begin{align*}
    \mathcal{G}_i(t) &\coloneqq \left\{|\hat{\mu}_i(t)-\mu_i|>\min\left\{1,\sqrt{\frac{2+6\log T}{N_i(t)}}\right\}\right\} \\
    \mathcal{G}(t) &\coloneqq \{\exists i\in[m]: \mathcal{G}_i(t)\} \\
    \mathcal{H}_i(t) &\coloneqq \{ i\in\tilde{S}(t),N_i(t)\leq(1-\rho)p_iM_i(t)\} \\
    \mathcal{H}(t) &\coloneqq \{\exists i\in[m]: \mathcal{H}_i(t)\} ~.
\end{align*}

\subsection{Facts and Lemmas}

\begin{fact}{(Proposition 1 in \cite{russo2014learning})} \label{fct:decomp}
    For any upper confidence bound sequence $U(S,t)$,
    \begin{align*}
        \mathrm{BayReg}(T) &= \Ex\left[ \sum_{t=1}^T(U(S(t),t)-r(S(t),\bm{\mu})) \right] \\
        &\hspace{36pt} + \Ex\left[ \sum_{t=1}^T(r(S^*,\bm{\mu})-U(S^*,t)) \right] ~.
    \end{align*}
\end{fact}
\begin{proof}
    Since $\bm{\theta}(t)$ is sampled from the posterior distribution of $\bm{\mu}$ given observation history $\mathcal{F}_{t-1}$, $S(t)=\mathrm{Oracle}(\bm{\theta}(t))$ and $S^*=\mathrm{Oracle}(\bm{\mu})$ follow the same distribution when  conditioned on $\mathcal{F}_{t-1}$. Together with the fact that $U(S,t)$ is a deterministic function when conditioned on $\mathcal{F}_{t-1}$, we have
    \begin{align*}
        &\Ex[r(S^*,\bm{\mu})-r(S(t),\bm{\mu})] \\
        &= \Ex[\Ex[r(S^*,\bm{\mu})-r(S(t),\bm{\mu})|\mathcal{F}_{t-1}]] \\
        &= \Ex[\Ex[r(S^*,\bm{\mu})-U(S^*,t)+U(S^*,t)-r(S(t),\bm{\mu})|\mathcal{F}_{t-1}]] \\
        &= \Ex[\Ex[r(S^*,\bm{\mu})-U(S^*,t)+U(S(t),t)-r(S(t),\bm{\mu})|\mathcal{F}_{t-1}]] \\
        &= \Ex[r(S^*,\bm{\mu})-U(S^*,t)]+\Ex[U(S(t),t)-r(S(t),\bm{\mu})] ~.
    \end{align*}
    for all $t\in[T]$.
\end{proof}

\begin{fact}{(Lemma 1 in \cite{russo2014learning})} \label{fct:eventg}
    \begin{align*}
        \Pr\left( \bigcup_{t=1}^T\left\{|\hat{\mu}_i(t)-\mu_i|>\min\left\{1,\sqrt{\frac{2+6\log T}{N_i(t)}}\right\}\right\} \right) \leq \frac{1}{T}
    \end{align*}
\end{fact}

\begin{fact}{(Multiplicative Chernoff bound \cite{chen2016combinatorial,mitzenmacher2005probability})} \label{fct:chernoff}
    Let $X_1,\ldots,X_n$ be Bernoulli random variables taking values in $\{0,1\}$ such that $\Ex[X_t|X_1,\ldots,X_{t-1}]\geq\mu$ for all $t\leq n$, and $Y=X_1+\cdots+X_n$. Then, for all $\delta\in(0,1)$,
    \begin{align*}
        \Pr(Y\leq(1-\delta)\mu n) \leq e^{-\frac{\delta^2\mu n}{2}} ~.
    \end{align*}
\end{fact}

\begin{lemma} \label{lmm:eventh}
    When CTS is run, we have
    \begin{align*}
        \Ex\left[ \sum_{t=1}^T\In\{i\in\tilde{S}(t),N_i(t)\leq(1-\rho)p_iM_i(t)\} \middle|\bm{\mu} \right] \leq 1+\frac{2}{\rho^2p^*}
    \end{align*}
    for all $i\in[m]$, $\bm{\mu}\in[0,1]^m$, and $\rho\in(0,1)$.
\end{lemma}
\begin{proof}
    \begin{align}
        \hspace{12pt}&\hspace{-12pt} \Ex\left[ \sum_{t=1}^T\In\{i\in\tilde{S}(t),N_i(t)\leq (1-\rho)p_iM_i(t)\} \middle| \bm{\mu}\right] \nonumber \\
        &\leq \Ex\left[ \sum_{w=0}^T\sum_{t=\tau_w^i+1}^{\tau_{w+1}^i}\mathbb{I}\{i\in\tilde{S}(t), \right. \nonumber \\[-18pt]
        &\hspace{96pt} N_i(t)\leq (1-\rho)p_iM_i(t)\} \Bigg| \bm{\mu} \Bigg] \nonumber \\
        &\leq \Ex\left[ \sum_{w=0}^T\In\{N_i(\tau_{w+1}^i)\leq (1-\rho)p_iM_i(\tau_{w+1}^i)\} \middle| \bm{\mu} \right] \nonumber \\
        &\leq 1 + \sum_{w=1}^T \Pr(N_i(\tau_{w+1}^i)\leq (1-\rho)p_iM_i(\tau_{w+1}^i)|\bm{\mu}) \nonumber \\
        &\leq 1 + \sum_{w=1}^T e^{-\frac{\rho^2p^*w}{2}} \label{eqn:eventh-a} \\
        &\leq 1 + \frac{2}{\rho^2p^*} \nonumber
    \end{align}
    where \eqref{eqn:eventh-a} is due to Fact~\ref{fct:chernoff}.
\end{proof}

\subsection{Main Part of the Proof} \label{sec:decomp}

We decompose the Bayesian regret as
\begin{align}
    \hspace{3pt}&\hspace{-3pt} \mathrm{BayReg}(T) \nonumber \\
    &= \Ex\left[ \sum_{t=1}^T(r(S(t),\bm{\bar{\mu}}(t))-r(S(t),\bm{\mu})) \right] \nonumber \\
    &\hspace{36pt} + \Ex\left[ \sum_{t=1}^T(r(S^*,\bm{\mu})-r(S^*,\bm{\bar{\mu}}(t))) \right] \label{eqn:decomp} \\
    &\leq \Ex\left[ \sum_{t=1}^T\In\{\neg\mathcal{G}(t),\neg\mathcal{H}(t)\}(r(S(t),\bm{\bar{\mu}}(t))-r(S(t),\bm{\mu})) \right] \label{eqn:decomp-a} \\
    &\hspace{6pt} + \Ex\left[ \sum_{t=1}^T\In\{\neg\mathcal{G}(t),\neg\mathcal{H}(t)\}(r(S^*,\bm{\mu})-r(S^*,\bm{\bar{\mu}}(t))) \right] \label{eqn:decomp-b} \\  
    &\hspace{6pt} + \Ex\left[ \sum_{t=1}^T\In\{\mathcal{G}(t)\vee\mathcal{H}(t)\}\right] \times 4mB \label{eqn:decomp-c} ~,
\end{align}
where \eqref{eqn:decomp} is due to Fact~\ref{fct:decomp}, and \eqref{eqn:decomp-c} is obtained by observing
\begin{align*}
    \hspace{12pt}&\hspace{-12pt} |r(S,\bm{\mu})-r(S,\bm{\bar{\mu}}(t))| \\
    &\leq B\sum_{i\in\tilde{S}}\left|\mu_i-\hat{\mu}_i(t)-\min\left\{1,\sqrt{\frac{2+6\log T}{N_i(t)}}\right\}\right| \\
    &\leq B\sum_{i\in\tilde{S}}|\mu_i-\hat{\mu}_i(t)| \\[-18pt]
    &\hspace{66pt} + B\sum_{i\in\tilde{S}}\min\left\{1,\sqrt{\frac{2+6\log T}{N_i(t)}}\right\} \\
    &\leq 2mB
\end{align*}
for all $S\in\mathcal{I}$.

\subsubsection{Bounding \eqref{eqn:decomp-a}}
When $\neg\mathcal{G}(t)$ and $\neg\mathcal{H}(t)$ hold, we have
\begin{align}
    \hspace{12pt}&\hspace{-12pt} r(S(t),\bm{\bar{\mu}}(t))-r(S(t),\bm{\mu}) \nonumber \\
    &\leq B\sum_{i\in\tilde{S}(t)}\left|\mu_i-\hat{\mu}_i(t)-\min\left\{1,\sqrt{\frac{2+6\log T}{N_i(t)}}\right\}\right| \nonumber \\
    &\leq B\sum_{i\in\tilde{S}(t)}|\mu_i-\hat{\mu}_i(t)| \nonumber \\[-18pt]
    &\hspace{66pt} + B\sum_{i\in\tilde{S}(t)}\min\left\{1,\sqrt{\frac{2+6\log T}{N_i(t)}}\right\} \nonumber \\
    &\leq 2B\sum_{i\in\tilde{S}(t)}\min\left\{1,\sqrt{\frac{2+6\log T}{N_i(t)}}\right\} \label{eqn:decomp-aa} \\
    &\leq 2B\sum_{i\in\tilde{S}(t)}\min\left\{1,\sqrt{\frac{2+6\log T}{(1-\rho)p_iM_i(t)}}\right\} \label{eqn:decomp-ab} ~,
\end{align}
where \eqref{eqn:decomp-aa} is due to $\neg\mathcal{G}(t)$ and \eqref{eqn:decomp-ab} is due to $\neg\mathcal{H}(t)$. Then,
\begin{align}
    \eqref{eqn:decomp-a} &\leq \Ex\left[ \sum_{t=1}^T \In\{\neg\mathcal{H}(t)\} 2B\sum_{i\in\tilde{S}(t)} \sqrt{\frac{2+6\log T}{(1-\rho)p_iM_i(t)}} \right] \nonumber \\
    &\leq \Ex\left[ \sum_{i=1}^m\sum_{w=0}^T\sum_{t=\tau_w^i+1}^{\tau_{w+1}^i} \In\{i\in\tilde{S}(t),\neg\mathcal{H}(t)\} \right. \nonumber \\[-12pt]
    &\hspace{114pt} \left. \times 2B\sqrt{\frac{2+6\log T}{(1-\rho)p_iM_i(t)}} \right] \nonumber \\
    &\leq \Ex\left[ \sum_{i=1}^m\sum_{w=0}^T \In\{\neg\mathcal{H}(\tau_{w+1}^i)\} 2B\sqrt{\frac{2+6\log T}{(1-\rho)p_iM_i(\tau_{w+1}^i)}} \right] \nonumber \\
    &\leq \Ex\left[ \sum_{i=1}^m\sum_{w=1}^T 2B\sqrt{\frac{2+6\log T}{(1-\rho)p_iM_i(\tau_{w+1}^i)}} \right] \label{eqn:decomp-ac} \\
    &\leq \sum_{i=1}^m\sum_{w=1}^T 2B\sqrt{\frac{2+6\log T}{(1-\rho)w}} \Ex_{\bm{\mu}}\left[\sqrt{\frac{1}{p^*}}\right] \nonumber \\
    &\leq 4mB\sqrt{\frac{T(2+6\log T)}{(1-\rho)}} \Ex_{\bm{\mu}}\left[\sqrt{\frac{1}{p^*}}\right] \label{eqn:decomp-ad} ~,
\end{align}
where \eqref{eqn:decomp-ac} holds since $N_i(\tau_1^i)=M_i(\tau_1^i)=0$ implies $\mathcal{H}(\tau_1^i)$ and \eqref{eqn:decomp-ad} holds since $\sum_{n=1}^N\sqrt{1/n}\leq2\sqrt{N}$.

\subsubsection{Bounding \eqref{eqn:decomp-b}} \label{sec:decomp-b}
When $\neg\mathcal{G}(t)$ holds, we have
\begin{align*}
    \mu_i \leq \hat{\mu}_i(t)+\min\left\{1,\sqrt{\frac{2+6\log T}{N_i(t)}}\right\} = \bar{\mu}(t)
\end{align*}
for all $i\in[m]$. Then,
\begin{align}
    r(S^*,\bm{\mu})-r(S^*,\bm{\bar{\mu}}(t)) &\leq r(S^*,\bm{\bar{\mu}}(t))-r(S^*,\bm{\bar{\mu}}(t)) \label{eqn:decomp-ba} \\
    &= 0 \nonumber ~,
\end{align}
where \eqref{eqn:decomp-ba} is due to Assumption~\ref{asm:rmonotonicity}. Hence, $\eqref{eqn:decomp-b}\leq 0$.

\subsubsection{Bounding \eqref{eqn:decomp-c}}
We have
\begin{align}
    \eqref{eqn:decomp-c} &\leq 4mB\sum_{i=1}^m\left( \Ex\left[\sum_{t=1}^T\In\{\mathcal{G}_i(t)\}\right]+\Ex\left[\sum_{t=1}^T\In\{\mathcal{H}_i(t)\}\right] \right) \nonumber \\
    &\leq 4mB\sum_{i=1}^m\left( T\Pr\left(\bigcup_{t=1}^T\{\mathcal{G}_i(t)\}\right) \right. \nonumber \\
    &\hspace{108pt} \left. + \Ex_{\bm{\mu}}\left[\Ex\left[\sum_{t=1}^T\In\{\mathcal{H}_i(t)\}\middle|\bm{\mu}\right]\right] \right) \nonumber \\
    &\leq 8m^2B\left(1+\frac{1}{\rho^2}\Ex_{\bm{\mu}}\left[\frac{1}{p^*}\right]\right) \label{eqn:decomp-ca} ~,
\end{align}
where \eqref{eqn:decomp-ca} is due to Fact~\ref{fct:eventg} and Lemma~\ref{lmm:eventh} respectively for the two terms.

\section{\reva{Proof of Theorem~\ref{thm:bayesiantpm}}}
\label{sec:proof-bayesiantpm}
In order to take advantage of Assumption~\ref{asm:rsmoothnesstpm}, we use the concept of triggering probability groups from \cite{wang2017improving}. However, the rest of our analysis is quite different from \cite{wang2017improving} and mainly follows the same technique we have followed in Section~\ref{sec:proof-bayesian} when proving Theorem~\ref{thm:bayesian}.

\subsection{Preliminaries}

In addition to the preliminaries in Section~\ref{sec:proof-bayesian-preliminaries} for the proof of Theorem~\ref{thm:bayesian}, we make the following definitions. For $j\in\mathbb{Z}_+$, let $\mathcal{I}_{i,j}\coloneqq\{S\in\mathcal{I}:2^{-j}<p_i^S\leq 2\cdot2^{-j}\}$ denote the $j$th \textit{triggering probability group} of base arm $i$ and let $j_i^S$ denote the index of the triggering probability group of base arm i that super arm $S$ belongs to, i.e., $j_i^S$ is such that $S\in\mathcal{I}_{i,j_i^{S}}$. We use these definitions to introduce the following counters: $M_{i,j}(t)\coloneqq\sum_{\tau=1}^{t-1}\In\{i\in\tilde{S}(\tau),S(\tau)\in\mathcal{I}_{i,j}\}$ and $N_{i,j}(t)\coloneqq\sum_{\tau=1}^{t-1}\In\{i\in S'(\tau),S(\tau)\in\mathcal{I}_{i,j}\}$. By definition, $M_i(t)=\sum_{j=1}^{\infty}M_{i,j}(t)$ and $N_i(t)=\sum_{j=1}^{\infty}N_{i,j}(t)$.

Given a particular base arm $i\in[m]$, let $\eta_w^{i,j}$ be the round for which base arm $i$ is in the triggering set $\tilde{S}(t)$ of the selected super arm $S(t)$ and $S(t)\in\mathcal{I}_{i,j}$ for the $w$th time and let $\eta_0^{i,j}=0$. Note that we have $i\in\tilde{S}(\eta_{w+1}^{i,j})$, $M_{i,j}(\eta_{w+1}^{i,j})=w$, and $S(\eta_{w+1}^{i,j})\in\mathcal{I}_{i,j}$ for all $w\geq 0$. We also make the following change to event $\mathcal{H}_i(t)$:
\begin{align*}
    \mathcal{H}_i(t) &\coloneqq \left\{\max\{N_i(t),8\log T\}\leq\frac{1}{2}\cdot 2^{-j_i^{S(t)}}M_{i,j_i^{S(t)}}(t)\right\} ~.
\end{align*}

\subsection{Facts and Lemmas}

\begin{lemma} \label{lmm:eventhtpm}
        Fix $i\in[m]$, $t\in[T]$ and $j\in\mathbb{Z}_+$. When CTS is run, we have
    \begin{align*}
        \Pr\left( \max\{N_i(t),8\log T\}\leq\frac{1}{2}\cdot 2^{-j}M_{i,j}(t) \right) \leq \frac{1}{T} ~.
    \end{align*}
\end{lemma}
\begin{proof}
    \begin{align}
        \hspace{12pt}&\hspace{-12pt} \Pr\left( \max\{N_i(t),8\log T\}\leq\frac{1}{2}\cdot 2^{-j}M_{i,j}(t) \right) \nonumber \\
        &\leq \Pr\left( N_i(t)\leq\frac{1}{2}\cdot 2^{-j}M_{i,j}(t) \middle| 8\log T\leq\frac{1}{2}\cdot 2^{-j}M_{i,j}(t) \right) \nonumber \\
        &\leq \sum_{w=0}^{T-1} \In\left( 8\log T\leq\frac{1}{2}\cdot 2^{-j}w\right) \nonumber \\[-6pt]
        &\hspace{48pt} \times \Pr\left( N_{i,j}(t)\leq\frac{1}{2}\cdot 2^{-j}w \middle| M_{i,j}(t)=w \right) \nonumber \\
        &\leq \sum_{w=0}^{T-1} \In\left( 8\log T\leq\frac{1}{2}\cdot 2^{-j}w\right) e^{-\frac{2^{-j}w}{8}} \label{eqn:eventhtpm-a} \\
        &\leq \sum_{w=0}^{T-1} e^{-2\log T} \label{eqn:eventhtpm-b} \\
        &\leq \frac{1}{T} \nonumber
    \end{align}
    where \eqref{eqn:eventhtpm-a} holds due to Fact~\ref{fct:chernoff} and \eqref{eqn:eventhtpm-b} holds since $8\log T\leq 1/2\cdot 2^{-j}w$ implies that $e^{-2^{-j}w/8}\leq e^{-2\log T}$.
\end{proof}

\subsection{Main Part of the Proof}

We decompose the Bayesian regret the same way as we did in Section~\ref{sec:decomp}. Note that \eqref{eqn:decomp-c} still holds since
\begin{align*}
    \hspace{12pt}&\hspace{-12pt} |r(S,\bm{\mu})-r(S,\bm{\bar{\mu}}(t))| \\
    &\leq B'\sum_{i\in\tilde{S}}p_i^S\left|\mu_i-\hat{\mu}_i(t)-\min\left\{1,\sqrt{\frac{2+6\log T}{N_i(t)}}\right\}\right| \\
    &\leq B'\sum_{i\in\tilde{S}}p_i^S|\mu_i-\hat{\mu}_i(t)| \\[-18pt]
    &\hspace{66pt} + B'\sum_{i\in\tilde{S}}p_i^S\min\left\{1,\sqrt{\frac{2+6\log T}{N_i(t)}}\right\} \\
    &\leq 2mB'
\end{align*}
for all $S\in\mathcal{I}$.

\subsubsection{Bounding \eqref{eqn:decomp-a}}
When $\neg\mathcal{H}(t)$ holds, one of the following must be the case:
\begin{align*}
    8\log T \geq{}& \frac{1}{2}\cdot 2^{-j_i^{S(t)}}M_{i,j_i^{S(t)}}(t) \\
    \implies& 1 \leq \sqrt{\frac{16\log T}{2^{-j_i^{S(t)}}M_{i,j_i^{S(t)}}(t)}} \leq \sqrt{\frac{4+16\log T}{2^{-j_i^{S(t)}}M_{i,j_i^{S(t)}}(t)}} ~, \\
    N_i(t) \geq{}& \frac{1}{2}\cdot 2^{-j_i^{S(t)}}M_{i,j_i^{S(t)}}(t) \\
    \implies& \sqrt{\frac{2+6\log T}{N_i(t)}}\leq \sqrt{\frac{4+12\log T}{2^{-j_i^{S(t)}}M_{i,j_i^{S(t)}}(t)}} \\
    &\hphantom{1 \leq \sqrt{\frac{16\log T}{2^{-j_i^{S(t)}}M_{i,j_i^{S(t)}}(t)}}} \leq \sqrt{\frac{4+16\log T}{2^{-j_i^{S(t)}}M_{i,j_i^{S(t)}}(t)}} ~.
\end{align*}
Combining the two result together, we obtain
\begin{align}
    \min\left\{1,\sqrt{\frac{2+6\log T}{N_i(t)}}\right\} \leq \sqrt{\frac{4+16\log T}{2^{-j_i^{S(t)}}M_{i,j_i^{S(t)}}(t)}} \label{eqn:decomp-aatpm} ~.
\end{align}
When $\neg\mathcal{G}(t)$ also holds, we have
\begin{align}
    \hspace{6pt}&\hspace{-6pt} r(S(t),\bm{\bar{\mu}}(t))-r(S(t),\bm{\mu}) \nonumber \\
    &\leq B'\sum_{i\in\tilde{S}(t)}p_i^{S(t)}\left| \mu_i-\bar{\mu}_i(t)-\min\left\{1,\sqrt{\frac{2+6\log T}{N_i(t)}}\right\} \right| \nonumber \\
    &\leq B'\sum_{i\in\tilde{S}(t)}p_i^{S(t)}|\mu_i-\bar{\mu}_i(t)| \nonumber \\[-18pt]
    &\hspace{66pt} + B'\sum_{i\in\tilde{S}(t)}p_i^{S(t)}\min\left\{1,\sqrt{\frac{2+6\log T}{N_i(t)}}\right\} \nonumber \\
    &\leq 2B'\sum_{i\in\tilde{S}(t)}p_i^{S(t)}\min\left\{1,\sqrt{\frac{2+6\log T}{N_i(t)}}\right\} \label{eqn:decomp-abtpm} \\
    &\leq 2B'\sum_{i\in\tilde{S}(t)}p_i^{S(t)}\min\left\{1,\sqrt{\frac{4+16\log T}{2^{-j_i^{S(t)}}M_{i,j_i^{S(t)}}(t)}}\right\} \label{eqn:decomp-actpm} \\
    &= 4B'\sum_{i\in\tilde{S}(t)}\min\left\{2^{-j_i^{S(t)}},\sqrt{\frac{(4+16\log T)2^{-j_i^{S(t)}}}{M_{i,j_i^{S(t)}}(t)}}\right\} \label{eqn:decomp-adtpm} ~,
\end{align}
where \eqref{eqn:decomp-abtpm} is due to $\neg\mathcal{G}(t)$, \eqref{eqn:decomp-actpm} is due to \eqref{eqn:decomp-aatpm}, and \eqref{eqn:decomp-adtpm} holds since $p_i^{S(t)}\leq 2\cdot 2^{-j_i^{S(t)}}$. Then,
\begin{align}
    \hspace{3pt}&\hspace{-3pt} \eqref{eqn:decomp-a} \nonumber \\[-6pt]
    &\leq \Ex\left[ \sum_{t=1}^T 4B'\sum_{\mathclap{i\in\tilde{S}(t)}} \min\left\{2^{-j_i^{S(t)}}\!,\sqrt{\frac{(4+16\log T)2^{-j_i^{S(t)}}\!}{M_{i,j_i^{S(t)}}(t)}}\right\} \right] \nonumber \\
    &\leq \Ex\left[ \sum_{i=1}^m\sum_{j=1}^{\infty}\sum_{w=0}^T\sum_{t=\eta_w^{i,j}+1}^{\eta_{w+1}^{i,j}} \In\{i\in\tilde{S}(t),S(t)\in\mathcal{I}_{i,j}\} \right. \nonumber \\[-6pt]
    &\hspace{69pt} \times \left. 4B'\min\left\{2^{-j},\sqrt{\frac{(4+16\log T)2^{-j}}{M_{i,j}(t)}}\right\} \right] \nonumber \\
    &\leq \Ex\left[ \sum_{i=1}^m\sum_{j=1}^{\infty}\sum_{w=0}^T 4B'\min\left\{2^{-j},\sqrt{\frac{(4+16\log T)2^{-j}}{M_{i,j}(\eta_{w+1}^{i,j})}}\right\} \right] \nonumber \\
    &\leq \sum_{i=1}^m\sum_{j=1}^{\infty}\left( 4B'\cdot 2^{-j}+\sum_{w=1}^T 4B'\sqrt{\frac{(4+16\log T)2^{-j}}{w}}\right) \nonumber \\
    &\leq 4mB' + 16mB'\sum_{j=1}^{\infty}\sqrt{(1+4\log T)T\cdot 2^{-j}} \label{eqn:decomp-aetpm} \\
    &\leq 4mB' + 16mB'(1+\sqrt{2})\sqrt{(1+4\log T)T} \nonumber ~,
\end{align}
where \eqref{eqn:decomp-aetpm} holds since $\sum_{n=1}^N\sqrt{1/n}\leq2\sqrt{N}$.

\subsubsection{Bounding \eqref{eqn:decomp-b}}
We bound \eqref{eqn:decomp-b} the same way we did in Section~\ref{sec:decomp-b}.

\subsubsection{Bounding \eqref{eqn:decomp-c}}
We have
\begin{align}
    \hspace{3pt}&\hspace{-3pt} \eqref{eqn:decomp-c} \nonumber \\[-6pt]
    &\leq 4mB'\sum_{i=1}^m\left( \Ex\left[\sum_{t=1}^T\In\{\mathcal{G}_i(t)\}\right]+\Ex\left[\sum_{t=1}^T\In\{\mathcal{H}_i(t)\}\right] \right) \nonumber \\
    &\leq 4mB'\sum_{i=1}^m\left( T\Pr\left(\bigcup_{t=1}^T\{\mathcal{G}_i(t)\}\right) + \sum_{t=1}^T\Pr(\mathcal{H}_i(t)) \right) \nonumber \\
    &\leq 4mB'\sum_{i=1}^m\left( 1 + \sum_{t=1}^T\sum_{j=1}^{\infty}\Pr(\mathcal{H}_i(t)|j_i^{S(t)}\!=j)\Pr(j_i^{S(t)}\!=j) \right) \label{eqn:decomp-catpm} \\
    &\leq 4mB'\sum_{i=1}^m\left( 1 + \sum_{t=1}^T\sum_{j=1}^{\infty}\frac{1}{T}\Pr(j_i^{S(t)}=j) \right) \label{eqn:decomp-cbtpm} \\
    &\leq 8m^2B' \nonumber ~,
\end{align}
where \eqref{eqn:decomp-catpm} is due to Fact~\ref{fct:eventg} and \eqref{eqn:decomp-cbtpm} is due to Lemma~\ref{lmm:eventhtpm}.

\section{Numerical Results}
\label{sec:numerical}
In this section, we compare CTS with other state-of-the-art CMAB algorithms in three different applications: cascading bandits, probabilistic maximum coverage bandits, and influence maximization bandits introduced in Section~\ref{sec:networking}. We compare the performance of CTS with CUCB in \cite{chen2016combinatorial} in all settings. For the first two problems, we assume that all algorithms have access to an exact computation oracle that computes the estimated optimal super arm in each round. On the other hand, for the third problem, we assume that all algorithms use an approximation oracle. For cascading bandits only, we also compare CTS with algorithms specifically designed for this setting: CascadeKL-UCB in \cite{kveton2015cascading} and TS-Cascade in \cite{cheung2019thompson}. The former uses the principle of optimism under the face of uncertainty to compute Kullback-Leibler divergence based UCBs while the latter uses Thompson sampling with Gaussian posterior over the base arms. 

\subsection{Cascading Bandits}

We consider the disjunctive case with $V=100$, $W=20$ and $K=5$, and generate $p_{i,j}$s by sampling uniformly at random from $[0,1]$. We run both CTS and CUCB for $1600$ rounds, and report their regrets averaged over $1000$ runs in Fig.~\ref{fig:comparison}, where error bars represent the standard deviation of the regret (multiplied by 10 for visibility). In this setting CTS significantly outperforms CUCB by achieving a final regret that is no more than $5\%$ of the final regret of CUCB. Relatively bad performance of CUCB can be explained by excessive number of explorations due to the UCBs that stay high for a large number of rounds.

\begin{figure}
	\centering
    \includegraphics[height=2.1in]{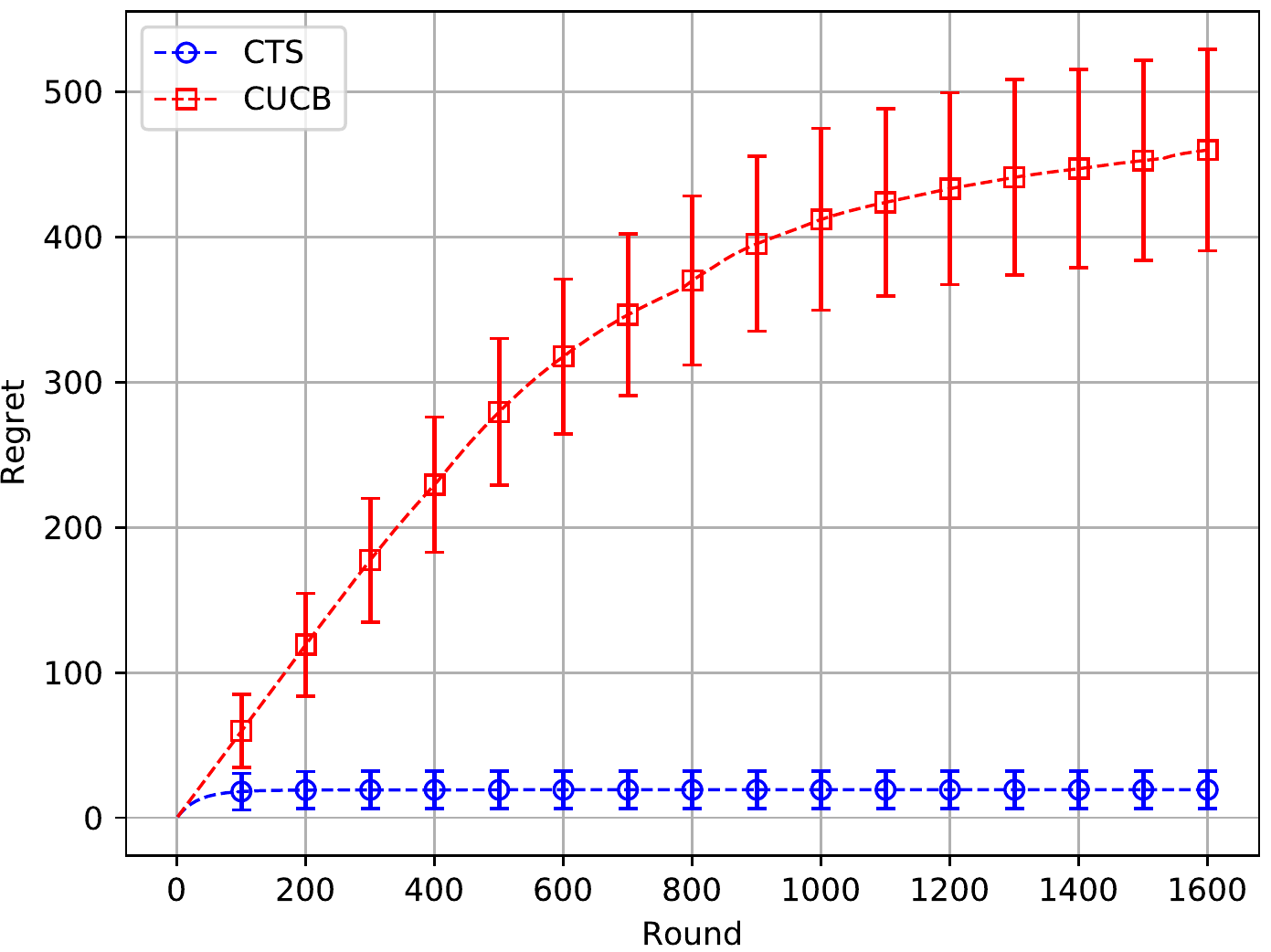}
	\caption{Regrets of CTS and CUCB for the disjunctive cascading bandit problem.}
	\label{fig:comparison}
\end{figure}

We also consider the same class of problems $B_{\mathrm{LB}}(V,K,p,\Delta)$ as in \cite{kveton2015cascading}, where $W=1$ and the probability that the user finds page $j$ attractive is given as
\begin{align*}
    p_{1,j}=\begin{dcases}
        p & \text{if } j\leq K \\
        p-\Delta & \text{otherwise} ~.
    \end{dcases}
\end{align*}
Similar to \cite{kveton2015cascading}, we set $p=0.2$ and vary other parameters, namely $V$, $K$, and $\Delta$. We run both CTS and CUCB for $100000$ rounds in all problem instances, and report their regrets averaged over $20$ runs in Table~\ref{tbl:results}.

\begin{table*}
    \caption{Regrets of CTS and CUCB with their standard deviations for various problem instances.}
    \label{tbl:results}
    \centering
    \begin{tabular}{*3{r}@{\hspace{24pt}}*5{r@{$\,\pm\,$}l}}
        \toprule
        $V$ & $K$ & $\Delta$ & \multicolumn{2}{c}{\textbf{CTS}} & \multicolumn{2}{c}{\textbf{CUCB}} & \multicolumn{2}{c}{\textbf{CascadeUCB1}} & \multicolumn{2}{c}{\textbf{CascadeKL-UCB}} & \multicolumn{2}{c}{\textbf{TS-Cascade}}\\
        \midrule
        16 & 2 & 0.15 & 155.4 & 14.1 & 1284.1 & 52.4 & 1300.6 & 46.8 & 360.6 & 23.4 & 381.1 & 16.8 \\
        16 & 4 & 0.15 & 103.2 & 9.0 & 998.9 & 33.2 & 993.6 & 32.8 & 267.3 & 20.6 & 281.0 & 11.8 \\
        16 & 8 & 0.15 & 52.1 & 9.8 & 549.5 & 16.8 & 546.4 & 11.7 & 150.3 & 15.6 & 137.9 & 8.8 \\
        32 & 2 & 0.15 & 321.4 & 18.9 & 2718.8 & 61.2 & 2676.4 & 59.4 & 749.2 & 34.2 & 752.9 & 49.9 \\
        32 & 4 & 0.15 & 252.2 & 17.0 & 2227.0 & 55.4 & 2232.1 & 46.6 & 617.4 & 39.9 & 612.3 & 15.2 \\
        32 & 8 & 0.15 & 155.4 & 25.7 & 1531.0 & 21.9 & 1525.4 & 30.0 & 420.6 & 27.5 & 385.0 & 16.3 \\
        16 & 2 & 0.075 & 276.9 & 50.7 & 2057.6 & 79.6 & 2065.4 & 87.4 & 709.0 & 60.4 & 688.3 & 78.5 \\
        16 & 4 & 0.075 & 205.4 & 25.7 & 1496.5 & 65.2 & 1512.4 & 87.0 & 546.3 & 53.5 & 557.9 & 45.0 \\
        16 & 8 & 0.075 & 113.1 & 40.4 & 719.4 & 53.7 & 717.5 & 44.2 & 266.1 & 32.4 & 273.8 & 30.7 \\
        \bottomrule
    \end{tabular}
\end{table*}

In addition to CUCB, we compare CTS against CascadeUCB1 and CascadeKL-UCB given in \cite{kveton2015cascading}, and TS-Cascade given in \cite{cheung2019thompson} as well. Note that regrets of CUCB and CascadeUCB1 matches very closely as two algorithms are essentially the same when CUCB is applied to cascading bandits except for some minor differences in the initialization stage and how UCBs larger than 1 are handled. We observe that CTS outperforms all other algorithms in all problem instances by achieving a regret that is at most $44\%$ of the regret of all other algorithms. For CTS, we also see that the regret increases as the number of pages ($V$) increases, it decreases as the number of recommended items ($K$) increases, and it increases as $\Delta$ decreases, which are very similar to the major observations that are made in \cite{kveton2015cascading}.

\subsection{Probabilistic Maximum Coverage Bandits}

Our experimental setup for this case is based on MovieLens dataset \cite{harper2015movielens} as in \cite{saritac2017combinatorial}.\footnote{While the probabilistic maximum coverage problem is NP-hard, here we focus on a small-scale problem and use an exact computation oracle.} The dataset contains 20 million movie ratings that are assigned between January 1995 and March 2015. Out of this, we only use the ones that are assigned between March 2014 and March 2015. In the experiments, the recommender chooses $K=3$ movies out of $V=30$ movies, which include $10$ of the most rated movies, $10$ of the least rated movies and $10$ randomly selected movies from the dataset. These $30$ movies are rated by $W=57369$ users. 

In total, there are $20$ genres in the dataset. Each movie belongs to at least one genre. We take genre information into account to define attraction probabilities. For this, we create a $20$-dimensional vector $\bm{g_i}$ for each movie $i\in[V]$, where $g_{ik}=1$ if the movie belongs to genre $k$ and $0$ otherwise. Using these vectors, we calculate a genre preference vector $\bm{u_j}$ for each user $j\in[W]$ as
\begin{align*}
    \bm{u_j} = \frac{\sum_{i\in\mathcal{V}_j}\bm{g_i}}{|\mathcal{V}_j|} + \bm{\epsilon_j}
\end{align*}
where $\mathcal{V}_j$ is the set of movies that user $j$ rated and $\bm{\epsilon_j}$ is a random vector such that $\epsilon_{jk} = |\chi_{jk}|$ for $\chi_{jk} \sim \mathcal{N}(0,0.05)$. The noise $\bm{\epsilon_j}$ is introduced to model exploratory behavior of the user. Finally, defining $\bm{\hat{g}_i}=\bm{g_i}/\|\bm{g_i}\|$ and $\bm{\hat{u}_j}=\bm{u_j}/\|\bm{u_j}\|$ as the normalized versions of the vectors we have defined, the attraction probabilities are calculated as
\begin{align*}
    p_{i,j} = 0.2 \times \frac{\langle\bm{\hat{g}_i},\bm{\hat{u}_j}\rangle r_i}{\max_{i\in[V]}r_i}
\end{align*}
where $r_i$ is the average rating of movie $i$.

We run both CTS and CUCB for $1000$ rounds, and report their regrets averaged over $10$ runs in Fig.~\ref{fig:movielens}, where error bars represent standard deviation of the regret (multiplied by 100 for visibility). We consider two cases with $p^*=0.01$ and $p^*=0.05$. For both cases, CTS significantly outperforms CUCB by achieving a final regret that is no more than $9\%$ of the final regret of CUCB.

\begin{figure}
	\centering	
    \includegraphics[height=2.1in]{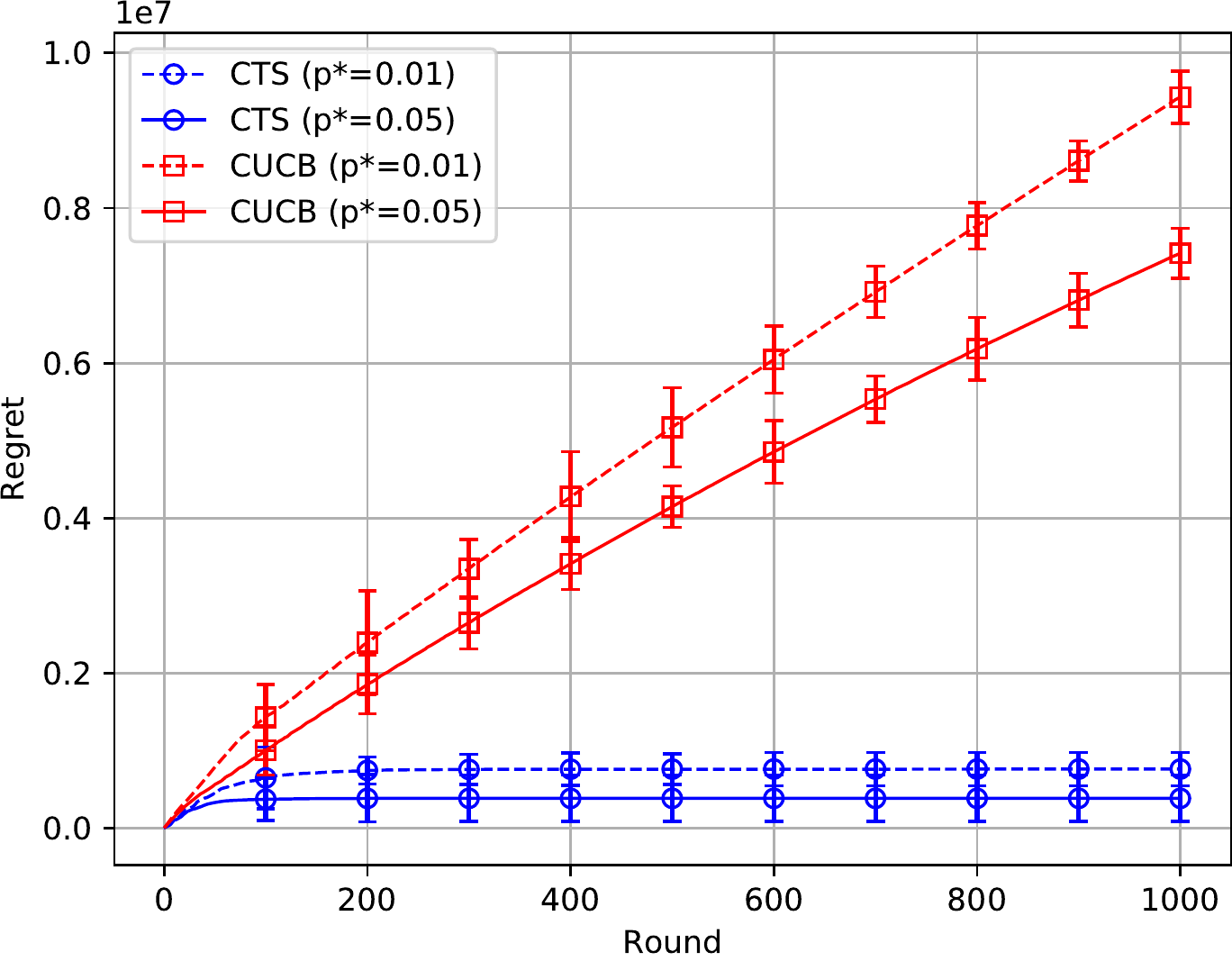}
	\caption{Regrets of CTS and CUCB for the probabilistic maximum coverage bandit problem.}
	\label{fig:movielens}
\end{figure}

\subsection{Influence Maximization Bandits}

We consider a directed version of the Facebook network dataset \cite{leskovec2012learning} that consists of $15$k edges and $3120$ nodes. Since, the dataset does not contain influence probabilities, we artificially generate them by setting  $p_{i,j}=1/|\mathcal{V}_i|$ where $\mathcal{V}_i$ represents the set of outgoing neighbors of node $i$. We assume that in each round the learner selects a seed set of $K=30$ nodes and this set forms the selected super arm. Moreover, we assume that the influence propagates---starting from the seed set---according to the independent cascade model \cite{kempe2003maximizing}, which is one of the most widely used influence propagation models. We adopt the edge-level feedback model in which the learner both observes the set of influenced nodes and the influence outcomes of the outgoing edges of these nodes. 

Since the problem itself is NP-hard, an exact computation oracle is computationally infeasible for the given graph size. Nevertheless, many computationally efficient approximation algorithms exist for the influence maximization problem (see e.g., CELF in \cite{leskovec2007cost}, and TIM and TIM+ in \cite{tang2014influence}). Due to its computational efficiency and good performance in practice, we set the learner to use TIM+ as the approximation oracle. When given as input an influence graph with $n$ nodes and $m$ edges, the influence probabilities on these edges and parameters $\varepsilon$ and $\ell$, TIM+ is guaranteed to return an $\alpha = (1-1/e-\varepsilon)$-approximate solution with probability at least $\beta = 1-3n^{-\ell}$ and with time complexity $O( (K+\ell) (n+m) \log n /\varepsilon^2)$. For all experiments, we set $\varepsilon=0.1$ and $\ell=1$. Since the learner uses an approximation oracle, instead of the regret given in \eqref{eqn:regret} we consider the $(\alpha,\beta)$-approximation regret as given in \cite{chen2016combinatorial} in the remainder of this section.

We run both CTS and CUCB for $5000$ rounds and report their regrets averaged over $10$ runs in Fig.~\ref{fig:facebook}. Here, error bars represent standard deviation of the regret multiplied by $10$ for visibility. Note that in these simulations, we consider the realized regret of the learner's actions instead of the expected regret as we do in the other experiments. This is once again due to the complexity of the problem and the difficulty in calculating expected regret. Again, it is observed that CTS significantly outperforms CUCB by achieving a final regret that is no more than $16\%$ of the final regret of CUCB. Relatively bad performance of CUCB is due to the fact that the considered time horizon is not long enough for CUCB to efficiently explore all base arms. It is observed that the UCBs of many base arms remain above $1$ even at the end of $5000$ rounds. As an algorithm that is based on the principle of optimism in the face of uncertainty, CUCB's performance completely depends on the confidence sets it uses to calculate the UCB indices, and this example shows that these confidence sets are not tight enough to guarantee fast convergence.

\begin{figure}
	\centering
    \includegraphics[height=2.1in]{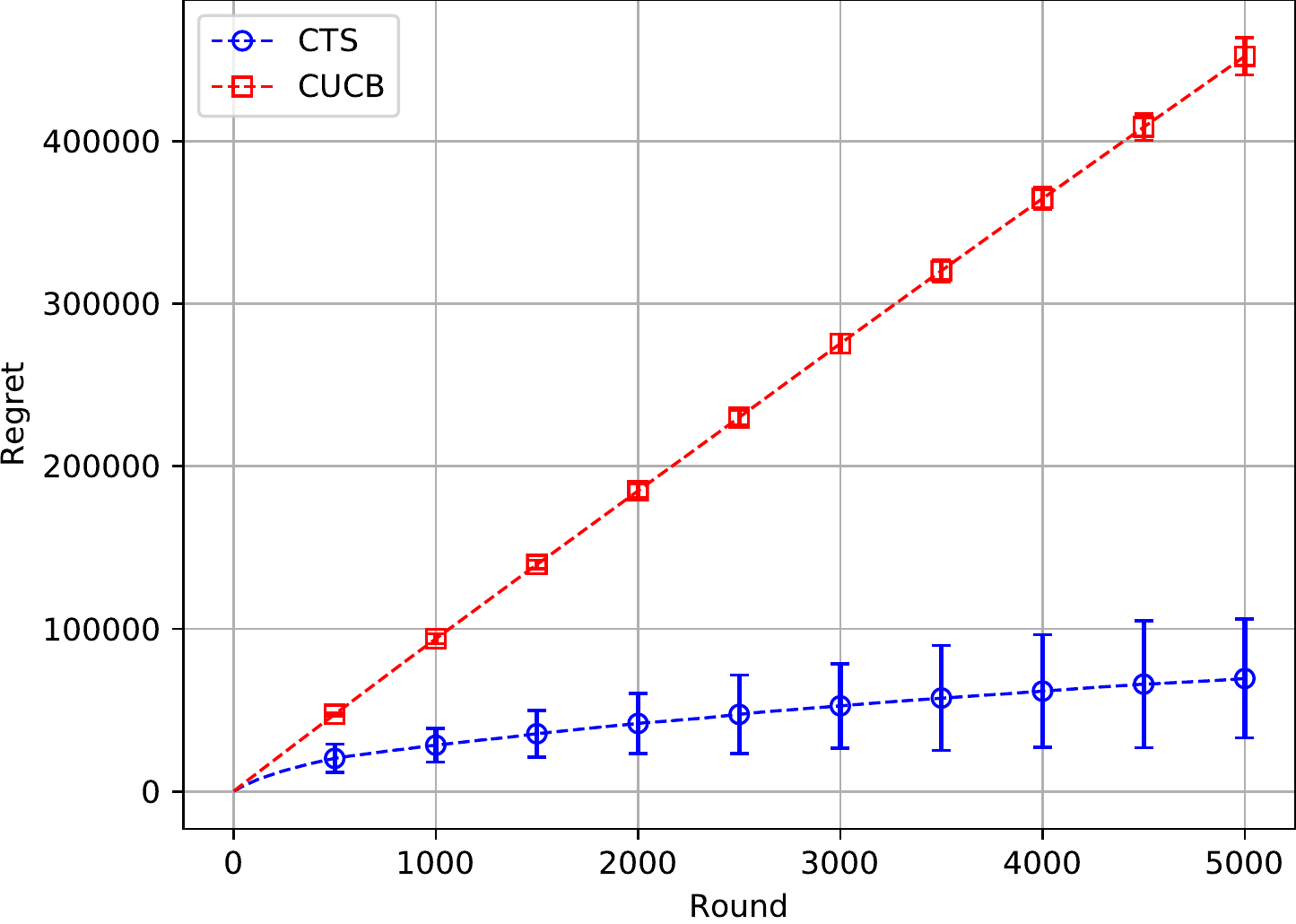}
	\caption{Regrets of CTS and CUCB for the influence maximization bandit problem.}
	\label{fig:facebook}
\end{figure}

\section{Conclusion}
\label{sec:conclusion}
\reva{We analyzed the regret of CTS for CMAB-PTA and proved (i) an order optimal gap-dependent regret bound when the expected reward function is Lipschitz continuous without assuming monotonicity, (ii) a Bayesian regret bound equivalent to an asymptotic gap-free regret bound assuming monotonicity, (iii) a Bayesian regret bound that is independent of triggering probabilities under the triggering modulated Lipschitz continuity assumption, and (iv) an improved regret bound that is independent of the time horizon for the special case when the triggering probabilities are strictly positive.}

\bibliographystyle{IEEEtran}
\bibliography{references}

%\newpage
%\appendices
%
%\section{Proof of Theorem~\ref{thm:main}}
%\label{sec:proof}
%\input{sources/proof}
%
%\section{Proof of Theorem~\ref{thm:positive}}
%\label{sec:proof-positive}
%\input{sources/proof-positive}
%
%\bibliographystyleappendix{IEEEtran}
%\bibliographyappendix{references}

\end{document}